\documentclass[sigconf]{acmart}
\AtBeginDocument{%
  }

\usepackage{amsmath,amsfonts}
\usepackage{marginnote}

\usepackage[flushleft]{threeparttable}
\usepackage{graphicx}
\usepackage{textcomp}
\usepackage{xcolor}
\usepackage{url}
\usepackage{enumitem}
\usepackage[lined,linesnumbered,boxed,commentsnumbered, ruled,vlined]{algorithm2e}
\usepackage{soul}
\usepackage{svg}

\usepackage{booktabs}
\usepackage{etoolbox}
\makeatletter
\patchcmd\algocf@Vline{\vrule}{\vrule \kern-0.4pt}{}{}
\patchcmd\algocf@Vsline{\vrule}{\vrule \kern-0.4pt}{}{}
\makeatother
\usepackage{gensymb}
\usepackage{subfigure}
\usepackage{graphicx,epstopdf,url,color}
\usepackage{balance}
\usepackage{placeins}
\usepackage{amsthm}
\usepackage{mathrsfs}
\usepackage{algpseudocode}
\usepackage{hyperref}
 \usepackage{multirow}
\usepackage{makecell}
\usepackage[normalem]{ulem}

\usepackage{xcolor}
\let\svthefootnote\thefootnote
\newcommand\freefootnote[1]{%
  \let\thefootnote\relax%
  \footnotetext{#1}%
  \let\thefootnote\svthefootnote%
}

\newtheorem{theorem}{\textbf{Theorem}}

\newtheorem{definition}{\textbf{Definition}}

\newtheorem{lemma}{\textbf{Lemma}}

\newcommand{\nonl}{\renewcommand{\nl}{\let\nl\oldnl}}
\usepackage{xspace} 

\newcommand{\ie}{\emph{i.e.,}\xspace}
\newcommand{\eg}{\emph{e.g.,}\xspace}

\setcopyright{none}
\renewcommand\footnotetextcopyrightpermission[1]{}
\begin{document}

\title{HiRAD: A Flexible Large-Scale AGV Routing System}

\author{Yunjie Huang}
\email{yhuang863@connect.hkust-gz.edu.cn}
\author{Ruizhong Wu}
\author{Mengxuan Zhang}
\author{Frodo Kin Sun Chan}
\author{Yan Nei Law}
\author{Lei Li}

\renewcommand{\shortauthors}{Huang et al.}

\begin{abstract}
  Automatic Guided Vehicles (AGVs)substantially boost warehouse throughput, but routing large-scale AGV fleets remains challenging. Classical Multi-Agent Pathfinding solvers suffer from exploding combinatorial complexity and super-quadratic runtime, while relying on idealized grid or piecewise-linear motion models that mismatch real-world kinematics. Recent Reinforcement Learning (RL) solutions improve flexibility via decentralized agent policies but depend on discretized spatiotemporal representations, require millions of episodes to converge, and incur full-map observation at every step, which leads to large models, slow convergence, and high inference latency that violates real-time industrial control constraints. To address these bottlenecks, we propose HiRAD, a hierarchical RL framework for continuous-space AGV routing with real-time guarantees, with 1) a step-level spatiotemporal representation that translates continuous motion into a differentiable RL problem, 2) a hierarchical strategy that splits heading choice from velocity control to reduce the action space, and 3) an asynchronous event-driven decision pipeline further lowers inference complexity from $O(n^2)$ to $O(n)$ and cuts per-step latency by as much as 71 \%. Across random graphs and two warehouse maps, HiRAD reduces makespan by 45 \%–63 \% and shortens end-to-end runtime. 
\end{abstract}





\maketitle

\section{Introduction}
\label{sec:Introduction}
The rapid growth of e-commerce and AI-driven logistics has propelled Automated Guided Vehicles (AGV) to a critical role in warehouse operations, which transports items from storage to packaging zones to meet large-scale order demands \cite{wurman2008coordinating,bogue2016growth,dhaliwal2020rise}. A large AGV fleet can significantly cut labor costs \cite{custodio2020flexible}, especially during peak shopping events when billions of orders are processed daily.

\begin{figure}
    \centering
    \includegraphics[width=1.0\linewidth]{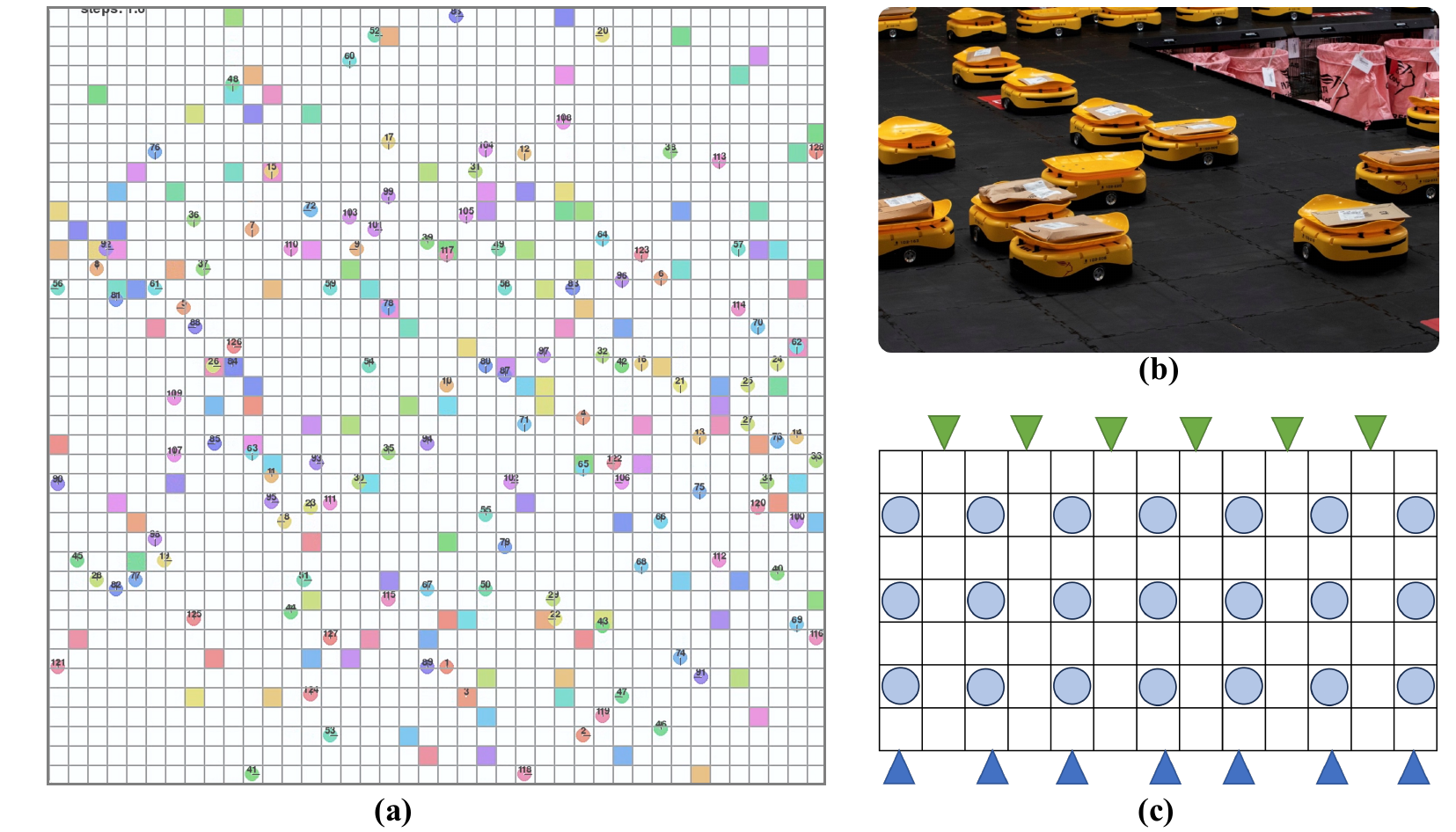}
    \vspace*{-4mm}
    \caption{(a) Snapshot of an AGV Operating Scene ; (b) Physical AGVs; (c) Warehouse Map (Green Triangles: Pickup locations, Blue Triangles: Dropoff locations).}
    \vspace*{-4mm}
    \label{fig:examples}
\end{figure}

Despite widespread deployment, AGV systems face fundamental limitations \cite{wu2025continuous,wu2025lifelong, boyarski2015icbs,li2021eecbs}, which must meet four base requirements: 1) \textit{Collision-free} operation as the safety baseline; 2) \textit{Flexibility} to handle ad-hoc failures and dynamic adjustments to tasks, layouts, or vehicles; 3) \textit{Kinematic Constraints} to adhere to physical movement laws in continuous environments; and 4) \textit{Efficiency} to complete tasks as quickly as possible. However, achieving both \textit{flexibility} and \textit{kinematic constraints} simultaneously remains a challenge.

To improve flexibility, routing algorithms must adapt quickly to environmental changes, but kinematic constraints complicate computations, often leading to problem simplifications in discrete environments: space is divided into grids, and AGVs move instantly between grids at discrete time intervals, sacrificing realism. This can even violate \textit{collision-free} guarantees; for instance, Figure \ref{fig:moving} illustrates a collision scenario caused by discrete assumptions that fail to capture real-world motion.

Conceptually, AGV routing is an NP-Hard \textit{Multi-Agent Path Finding (MAPF)} problem \cite{yu2013structure}. Early robotic solutions planned paths for all AGVs in advance to avoid collisions \cite{boyarski2015icbs,li2021eecbs}, but they were computationally slow and impractical for scaling. Faster pruning-based algorithms \cite{shicollision,shi2022adaptive}, like OHSMD \cite{wu2025continuous}, introduced continuous kinematic constraints but lacked \textit{flexibility}, as deviations in AGV paths required system-wide re-planning. To address \textit{flexibility}, reinforcement learning (RL) trains AGVs to independently find paths \cite{damani2021primal,maoudj2022decentralized,liu2020prediction}. However, RL methods scale poorly, requiring days to train for large fleets and sacrificing \textit{kinematic constraints}. PRIMAL \cite{primal} is the only RL-based solution to support over 1,000 AGVs, but it operates in discrete environments and treats collisions as penalties rather than strict constraints, limiting its practicality.

In summary, achieving both \textit{flexibility} and \textit{kinematic constraints} remains a key challenge for improving AGV efficiency. Since full-path planning conflicts with \textit{flexibility}, we adopt the RL-based paradigm for adaptability, extend it to satisfy kinematic constraints efficiently using data management techniques, and propose the \textbf{\uline{Hi}gh-\uline{R}esolution and \uline{Hi}erarchical \uline{R}outing with \uline{A}synchronous \uline{D}ecision Framework (HiRAD)} with the following componentes.

\begin{figure}[t]
    \centering
    \includegraphics[width=0.7\linewidth]{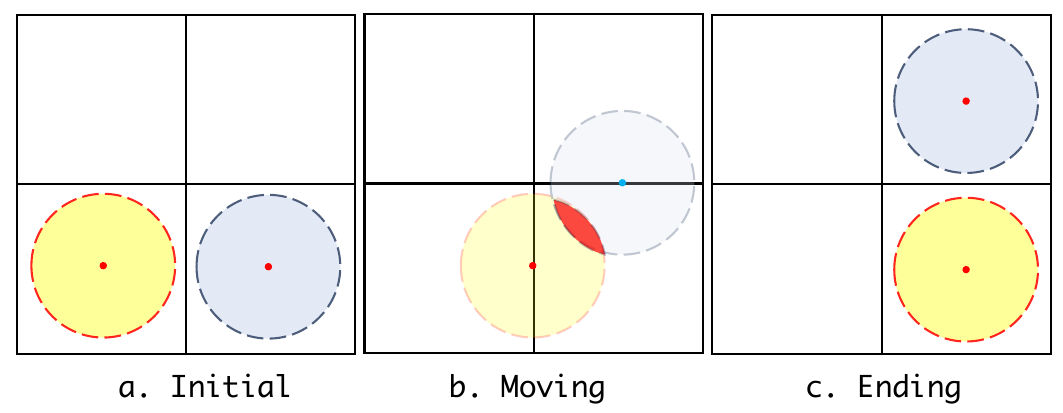}
    \caption{AGV Conflicts in Discrete and Continuous Moving}    
    \vspace*{-3mm}
    \label{fig:moving}
\end{figure}

Firstly, RL's step-based decision-making conflicts with continuous environments. To address this, we propose \textit{High-Resolution Spatiotemporal Modeling}, which uses finer-grained discrete space and time aligned with AGV kinematics, effectively simulating a continuous environment. Grids are divided into smaller cells, time into finer intervals, and AGV movements into intermediate velocity states, allowing smooth ac-/deceleration to mimic real-world dynamics. However, it inevitably enlarges the search space and decision-making time, while the higher number of steps demands faster decisions, making real-time operation even more challenging.

Secondly, to reduce the action search space, we propose \textit{Alternating Decision-Velocity Control (Alt-DVC)}, which first decouples direction and velocity decisions and reduces their transition sizes. The \textit{Macro-Routing Layer} operates on a coarse grid map using RL-based decision-making to generate global paths, while the \textit{Micro-Velocity Control Layer} runs on a fine-grained map with predefined velocity modes, reducing action frequency and computational cost by dynamically switching between modes. Additionally, \textit{Alt-DVC} becomes the first RL-based method to avoid collisions. \textit{Static constraints} treat AGVs as dynamic obstacles to prevent overlapping positions, and \textit{priority scheduling} resolves path conflicts by assigning right-of-way, reducing congestion and wait times. This approach minimizes collision risks, equipment wear, and task interruptions, significantly improving operational efficiency and safety.

Finally, the above constraints and priority handling make decision-making time larger the AGV's physical execution time, causing loss of control in practice. This arises from the high complexity of existing observation and priority scheduling methods. To address this, we propose \textit{Asynchronous Decision-Velocity Control (Asy-DVC)} optimized with \textit{Observation Pruning} and \textit{Map-Oriented Priority} to reduce complexity from quadratic to linear. This ensures decision-making remains faster than execution time, enabling efficient and seamless AGV control in real-world warehouse environments.

Our contributions are summarized below: 
\begin{list}{$\bullet$}{\leftmargin=1em \itemindent=0em}
    \item We propose the HiRAD to enable a large-scale AGVs collision-free routing system to work in kinematic constraint environments flexibly and efficiently;
    \item We propose a high-resolution spatiotemporal model that adapts the step-based RL framework to the continuous environment; 
    \item We propose a hierarchical Alternating Decision-Velocity Control to decouple the direction and speed decision-making to reduce the search space and avoid collisions;
    \item We propose Asynchronous Decision-Velocity Control with observation pruning and map-oriented priority to further reduce the decision-making complexity for real-life use;
    \item Evaluations with extensive experiments show that our approach outperforms the state-of-the-art solutions.  
\end{list}

\section{Related Work}
\label{sec:Related}
\subsection{Single-Agent Pathfinding}
\label{subsec:Related_SAPF}
These methods focus on single AGV path planning and differ in environmental information accessibility \cite{zhang2018path,liu2023path}. In \textit{local pathfinding}, the environment is partially known via LiDAR or visual sensors, resembling general robotics but not AGVs. 
\textit{Global pathfinding}, where the full environment is known, employs models like Voronoi diagrams \cite{barer2014suboptimal,wang2016voronoi,gomez2020hybrid,stenzel2021automated}, probabilistic maps \cite{geraerts2004comparative}, and geometric methods \cite{liang2018geometrical}. Sampling-based algorithms like \textit{RRT} \cite{lavalle1998rapidly} and \textit{RRT*} \cite{karaman2011sampling} incrementally build search spaces. For AGV warehouses, environments are modeled as \textit{grids} \cite{howden1968sofa} or \textit{topological graphs} 
\cite{chien1984planning}, enabling graph algorithms like \textit{Dijkstra's} \cite{dijkstra1959note} 
and $A^{*}$ \cite{hart1968formal}, and their dynamic variants
$D^{*}$ \cite{stentz1994optimal}, \textit{Lifelong} $A^{*}$ 
\cite{koenig2004lifelong}, and $D^{*}$-lite \cite{koenig2005fast}, to adjust for new obstacles, but they struggle with multi-agent, collision-free, optimal pathfinding due to slow computation. Path indexes like \textit{CH} 
\cite{geisberger2008contraction,ouyang2020efficient} and \textit{HL} \cite{cohen2003reachability,ouyang2018hierarchy,akiba2013fast,zhang2021dynamic,zhang2021efficient,zhang2021experimental} are fast but cannot handle collisions.

\subsection{Multi-Agent Pathfinding MAPF}
\label{subsec:Related_MAPF}
Multi-agent scenarios aim to coordinate AGV fleets efficiently without conflicts. A simple approach is to use single-agent algorithms, detect conflicts, and re-route dynamically \cite{bailey1990automated,dai2011artificial,riman2024novel}, but re-routing is time-consuming and reduces efficiency. Classical optimization techniques like \textit{Ant Colony} \cite{lissovoi2013runtime}, \textit{Particle Swarm} \cite{mohiuddin2016fuzzy}, and \textit{Genetic Algorithms} \cite{shorakaei2016optimal} are unsuitable for real-time MAPF due to slowness and batch processing. \textit{CBS} \cite{boyarski2015icbs,li2019improved,dai2011artificial} and \textit{ODrM*} \cite{ferner2013odrm} improve multi-agent search efficiency but require simultaneous planning for all agents, limiting online application. \textit{Cooperative A$^*$} \cite{6907401} plans paths individually but is restricted to discrete environments and lacks support for kinematics or lifelong learning. \textit{OHSMD} \cite{wu2025lifelong,wu2026Demo}, derived from time-dependent routing \cite{li2017minimal}, operates in continuous environments with kinematic constraints and lifelong scheduling, but its centralized system requires global re-planning after incidents, making it inflexible and impractical.
\subsection{RL-based Pathfinding}
\label{subsec:Related_RLPF}
Recent advancements have integrated deep reinforcement learning (DRL) and graph neural networks (GNNs) into pathfinding, offering greater flexibility by identifying optimal actions based on the current environment and single-agent training. \textit{CADRL} \cite{chen2017decentralized} replaced traditional online computations with offline-trained value networks for decentralized, non-communicating multi-robot settings. Ding \cite{ding2018hierarchical} introduced a hierarchical RL framework using LiDAR data, with high-level risk assessment and low-level action decisions, while \cite{long2018towards} addressed multi-robot collision avoidance using 512-dimensional LiDAR data, target positions, and robot velocities as inputs, with continuous translational and rotational velocities as outputs. However, these studies were limited to AGV fleets of 32 or fewer, where classical algorithms remain effective. The \textit{PRIMAL} framework \cite{primal} scaled to 1024 agents, and occasionally 2048, marking a milestone in large-scale MAPF with RL-based approaches. Despite this success, PRIMAL operates in discrete settings, limiting its ability to simulate real-world AGV motion dynamics.

\section{Problem Formulation}
In an AGV network, the environment is modeled as a grid map $M$ consisting of $m \times k$ square grids, where each AGV occupies a single grid at any given time.  AGVs move between grids to complete tasks $\tau_i = \langle s_i, d_i \rangle$, where $s_i$ and $d_i$ are the start and destination positions. AGVs can perform five actions: moving to adjacent grids or staying in place ($a^t =\{\uparrow, \rightarrow, \downarrow, \leftarrow, \odot$\}).
To account for kinematic constraints, we additionally associate each time step with a scalar velocity $v_i^t$ (under a high-resolution modeling). Thus, the objective of our Continuous Reinforcement Learning Multi-Agent Pathfinding (C-RL-MAPF) problem is defined to compute collision-free paths $\mathcal{P} = \{p_i = \langle (x_i^t, y_i^t, v_i^t) \rangle\}_{i=1}^{N}$ for all $N$ AGVs that minimize the makespan (maximum path length) and flowtime (total path length). $\langle (x_i^t, y_i^t) \rangle$ is the location of the $i$-th AGV at time step $t$ with the velocity $v_i^t$.
The discrete action specifies the intended neighboring cell, while $v_i^t$ determines how this intent is realized over $\Delta t$.

\section{High-Resolution Spatiotemporal Modeling}
\label{sec:HRS}

To balance the trade-off between the coarse-grained discrete modeling and the computationally expensive on continuous modeling, we propose the High-Resolution Spatiotemporal (HRS) modeling. HRS integrates a fine-grained spatial map (\textit{HR-Map}) with a finer temporal decision-making step (\textit{HR-Step}), enabling discrete modeling to approximate continuous environments efficiently.

\begin{figure}[t!]
    \center
    \includegraphics[width=1.0\linewidth]{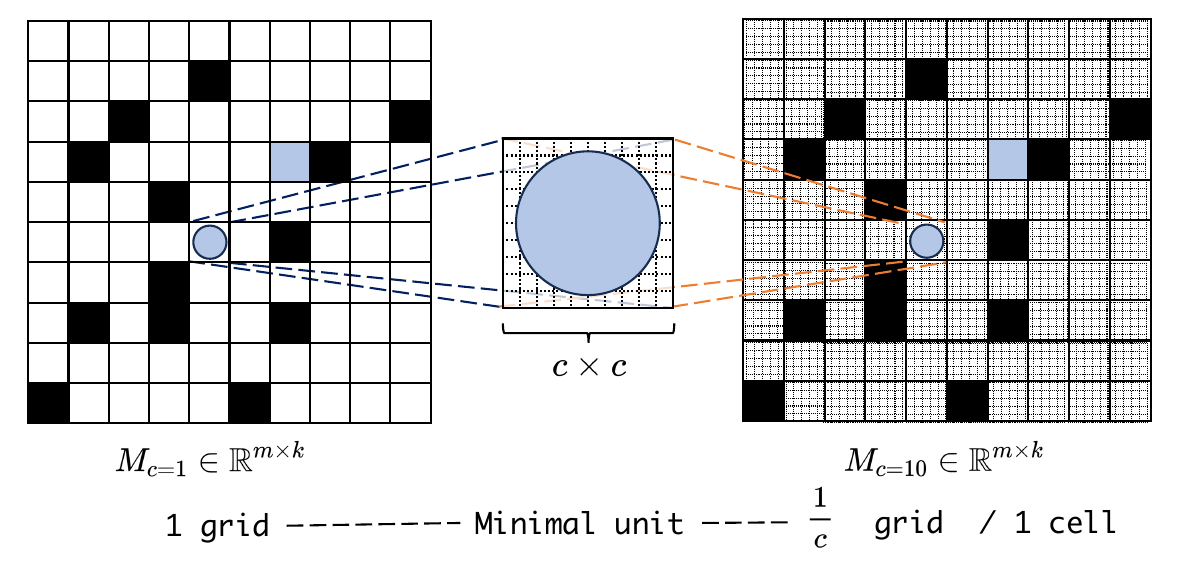}
    \caption{High-Resolution Spatiotemporal Modeling}
    \vspace*{-4mm}
    \label{fig:GridNetwork}
\end{figure}

\paragraph{High-Resolution Map and Step}
\label{subsec:HRS_Model}

The definitions are described as:
\begin{definition}[\textbf{High-Resolution HR-Map}]
\label{def:HRMap}
Given a grid network $M$, we increase the spatial resolution by dividing each grid into $c \times c$ smaller cells (as Fig. \ref{fig:GridNetwork}), resulting in an HR-Map $M_c \in \mathbb{R}^{(mc) \times (kc)}$, where $c > 1$. The original grid map $M = M_{c=1}$ is referred to as a Coarse Map.
\end{definition}

\begin{definition}[\textbf{High-Resolution HR-Step}]
\label{def:HRStep}
Given a decision interval $\delta$ in a coarse map, we divide it into $\Delta t = \frac{\delta}{c}$ HR-Steps such that AGVs make decisions at higher temporal resolution.
\end{definition}

In the HR-Map, AGVs can move between cells at varying speeds, which correspond to consecutive cell movements in HR-Steps: acceleration, constant velocity, or deceleration. This discrete yet fine-grained modeling is equivalent to continuous environments while maintaining computational efficiency.

\paragraph{HRS RL-MAPF Modeling}
\label{subsec:HRS_Modeling}

HRS RL-MAPF extends discrete RL-MAPF by incorporating velocity into the action space $a^t = (\mathbf{d}a^t, v^t)$, where $\mathbf{d}a^t$ represents direction choice in the classical MAPF problem and $v^t$ represents the number of cells moved per HR-Step. The velocity can increase, remain constant, or decrease between steps, reflecting acceleration and deceleration.

For a C-RL-MAPF where $c \rightarrow \infty$, the complexity for an RL-based method would increase from $(m\times k \times |\mathbf{d}a| )^{|\mathcal{N}|}$ to $(cm \times ck \times a \times |v|)^{|\mathcal{N}|}$ as the Expanded range of velocity $v^t$ and action directions adds complexity to the policy optimization.  The reward function $\mathcal{R}(s^t, a^t)$ also trapped in reducing positional error $e = \| (x^t, y^t) - (x^*, y^*) \|$ to ensure AGVs stop precisely at grid centers. 

Therefore, HRS modeling could achieve continuous, large-scale AGV routing while controlling complexity: instead of letting $c\to\infty$ and freely expanding $v^t$, HRS discretizes motion into a small set of kinematically feasible velocity modes and uses HR-step-based decisions to limit the effective action/velocity dimensions. In this way, HRS preserves continuous feasibility and coordination, while reducing both the spatiotemporal state growth and the joint action explosion compared with a naive continuous formulation.

\begin{figure*}[t]
    \centering
    \includegraphics[width=0.95\linewidth]{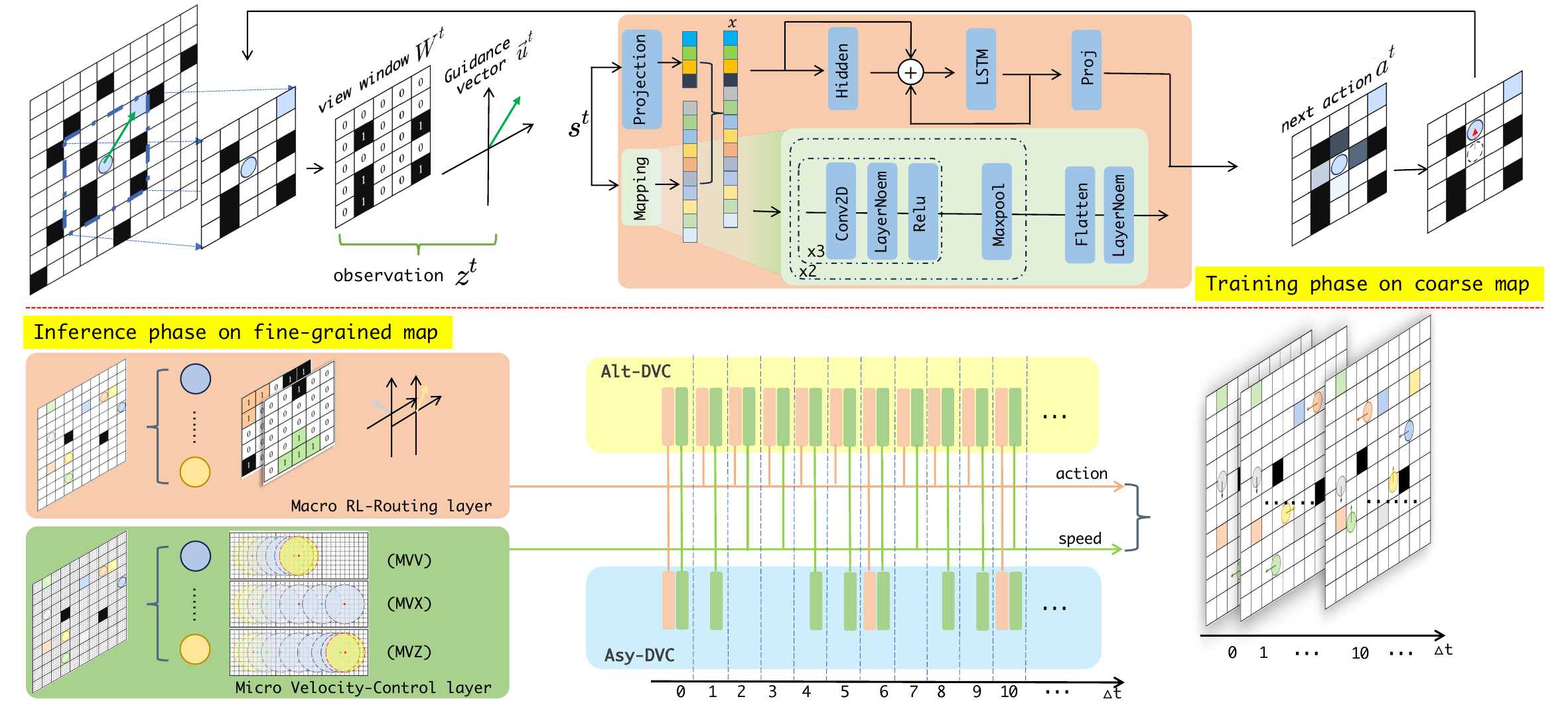}
    \caption{HiRAD: High-Resolution and Hierarchical Routing with Asynchronous Decision Framework.}
    \vspace*{-4mm}
    \label{fig:framework}
\end{figure*}

\section{HiRAD AGV Routing Framework}
\label{sec:hirad}
Figure~\ref{fig:framework} illustrates our HiRAD framework, which decouples AGV navigation into \emph{direction} and \emph{velocity} decisions. At the macro level, we propose a \textit{Macro RL-Routing Layer} (orange) that plans collision-free moving directions toward the goal using coarse-map information, where obstacles include both static structures and other AGVs treated as dynamic obstacles. At the micro level, we introduce a \textit{Micro Velocity-Control Layer} (green) to execute safe, fine-grained motions on the HR-Map. These two layers are integrated by the \textit{Alternating Decision--Velocity Control (Alt-DVC)} module (yellow) to drive the AGV, while the \textit{Asynchronous Decision--Velocity Control (Asy-DVC)} module (blue) further reduces decision frequency to improve inference efficiency. We next detail each component.

\subsection{Macro-RL-Routing Layer} 
\label{subsec:hirad_macro}
Macro-RL-Routing performs coarse-level navigation by outputting collision-free moving directions on the coarse grid. To avoid expensive multi-AGV training, we train a single-agent policy on randomly generated maps $M_{c=1}$ and reuse it as the macro routing module. At each step, the policy receives the partial observation state $z^t \approx s^t$ which includes a local window $W^t \in \mathbb{R}^{b\times b}$ and a guidance vector $\vec{u}^t \in \mathbb{R}^3$, where $b$ is the number of grids. Both of them are encoded and concatenated as the observation representation $x$, then fed into an LSTM to model temporal context. At last of module, a projection head produces the action distribution $\pi(a^t \mid s^t)$ (direction decision) and the state value $V(s^t)$ (critic signal). Training adopts curriculum learning (from small obstacle-free maps to large dense ones) with step/collision penalties and a difficulty-scaled terminal goal reward to encourage long-horizon routing. The specific training process and reward settings are displayed in Appendix \ref{appendix:hirad_macro_Training}.

We adopt the Macro-layer, which embeds the well-trained routing policy $\pi^*$ to generate the direction actions for all AGVs. Both the view window and the guidance vector are extended from the single-AGV form to accommodate all AGVs. The observation becomes $W \in \mathbb{R}^{|\mathcal{N}| \times b \times b}$, and the guidance vector becomes $\vec{u} \in \mathbb{R}^{|\mathcal{N}| \times 3}$, enabling simultaneous decision-making for the entire fleet.
As AGVs may traverse multiple fine-grained cells across grid boundaries under the HR-map, we must treat AGVs as dynamic obstacles occupying spatial regions. To handle this, we define AGV-specific \textit{exclusion zones} as follows:

\begin{definition}[\textbf{Exclusion Zone}]
    \label{def:Exclusion}
    The Exclusion Zone for a specific AGV $i$ is denoted as $\mathcal{E}_i \subseteq M_{c=1}$ ensures safe and efficient navigation, which consists of five grids: the two grids ahead in the AGV’s moving direction (active zone) and the grids to its left, right, and rear (passive zones). 
\end{definition}

The exclusion zone for each AGV consists of five grids that indicate its occupied and potentially occupied regions. These zones are used to proactively warn other AGVs of intended movements, enabling complete collision avoidance at the decision level.

\begin{figure}[t]
    \centering
    \includegraphics[width=1.0\linewidth]{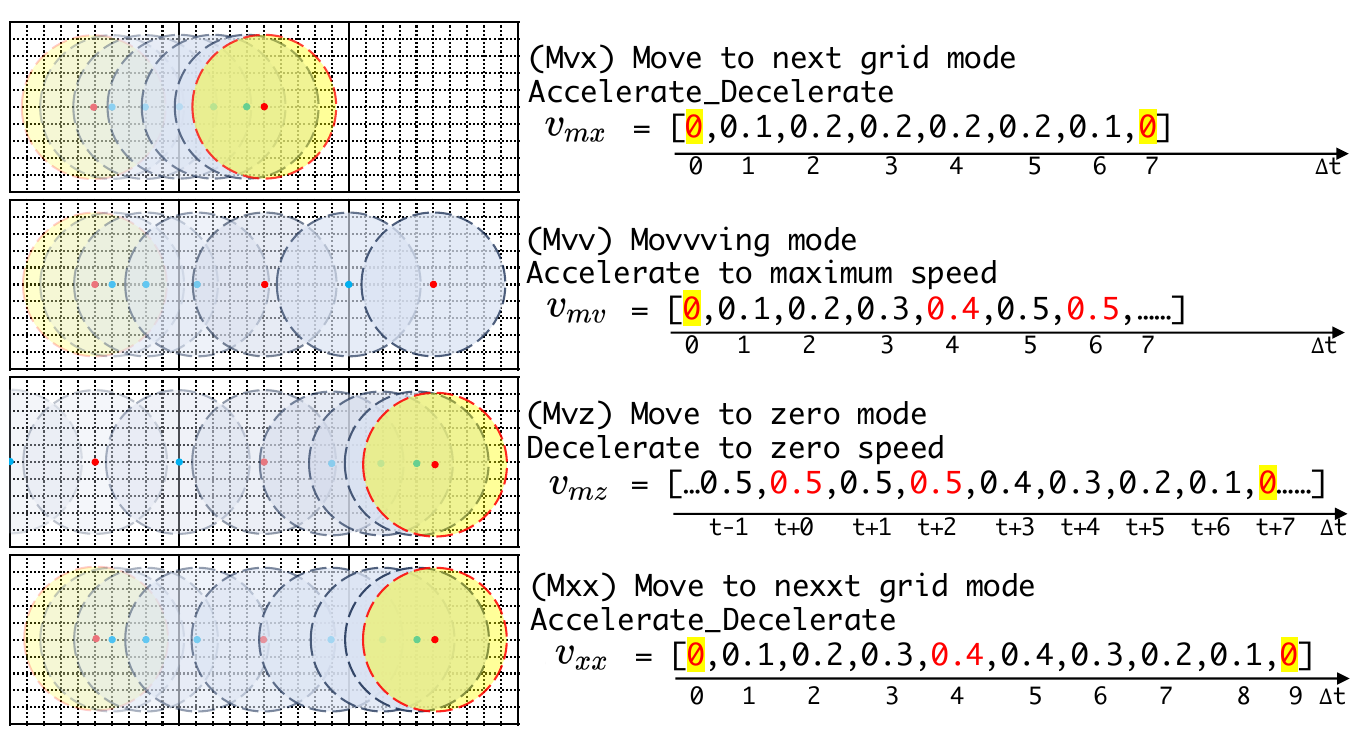}
    \caption{Velocity Mapping. The time units in yellow circles represent the HR-Steps during which the Macro-RL-Routing makes decisions. The time units in red, centered in circles, represent the HR-Steps to execute the Micro-Velocity-Control decisions for switching velocity modes.}
    \vspace*{-4mm}
    \label{fig:speedmode}
\end{figure}

\subsection{Micro Velocity-Control Layer}
\label{subsec:hirad_micro}
Micro Velocity-Control is a lightweight module that adjusts the AGV travel velocity online according to local motion requirements recommended by the Macro-Routing layer. We explain the fine-grained control on the HR-map with resolution factor $c=10$, where \emph{velocity} is mapped to the distance traveled within the minimum control interval $\Delta t$ (e.g., $v=0.2$ means moving $0.2$ grids, i.e., 2 HR-cells, per $\Delta t$). To satisfy kinematic constraints and enable precise stopping at grid centers, we predefine several velocity modes that follow a three-stage profile \textit{Acceleration--(Max)Speed--Deceleration}, covering typical cases such as moving to the next grid (Mvx), continuous cruising (Mvv), decelerating to stop (Mvz), and move to the second grid (Mxx) as Figure \ref{fig:speedmode}. 

Here, \textit{velocity} is converted to the distance an AGV can travel within the smallest time unit $\Delta t$. For example, a velocity of $0.2$ means that the AGV moves $0.2$ grids (\ie 2 HR-cells) per $\Delta t$, whereas a velocity of $0$ means staying at the grid center. The specific velocity distributions and behavioral characteristics of the three velocity modes are detailed as follows:

\begin{enumerate}[leftmargin=5mm]
    \item \textit{\textbf{Move to Next Grid}} (\textbf{Mvx}): This mode applies when the AGV needs to move only one grid. The velocity pattern is [0, 0.1, 0.2, 0.2, 0.2, 0.2, 0.1, 0], meaning the AGV slightly accelerates, then moves at a constant speed, and finally decelerates to a stop. It takes 7 $\Delta t$ to achieve the discrete movement;
    \item \textit{\textbf{Continuous Motion}} (\textbf{Mvv}): This mode is used when there is no obstacle ahead, allowing the AGV to accelerate to its maximum velocity and maintain high-velocity movement. The velocity pattern follows [0, 0.1, 0.2, 0.3, 0.4, 0.5, 0.5, ...], thereby improving operational efficiency. At the maximum velocity, it takes 2 $\Delta t$ to traverse a grid. In addition, it is guaranteed that at every HR-Step $\Delta t$ at 0.5 speed, the AGV is at the grid center or boundary;
    \item \textit{\textbf{Move to zero}} (\textbf{Mvz}): This mode is applied when the AGV needs to stop at the next grid after moving at maximum speed. The velocity pattern follows [..., 0.5, 0.5, 0.5, 0.5, 0.4, 0.3, 0.2, 0.1, 0], where the AGV maintains its peak speed initially and then decelerates uniformly to a full stop. The first velocity decision is made at time step $t+0$, where the AGV chooses to maintain its maximum velocity. Upon reaching the center of the next grid at $t+3$, it initiates deceleration, reducing speed at a constant rate to ensure it comes to still exactly at the grid center.
    \item \textit{\textbf{Move to Second Grid}} (\textbf{Mxx}): This mode applies when the AGV needs to stop at the second grid ahead. The velocity pattern is [0, 0.1, 0.2, 0.3, 0.4, 0.4, 0.3, 0.2, 0.1, 0], allowing for greater acceleration followed by a smooth deceleration to ensure stable movement and precise stopping. The AGV cannot reach its maximum speed in this mode.  
\end{enumerate}

Speed modes guaranteed real-time applicability conditions for HR-step-based decision making:
\begin{lemma}[\textbf{RL Real-Time Applicability Conditions}]
\label{lemma:Condition}
Given an HRS-Modeling and its corresponding velocity mapping, the AGVs can run continuously under kinematic constraints in real life only when the following conditions are satisfied:
\begin{enumerate}[leftmargin=5mm]
    \item The AGV stays at the grid center when it stops;
    \item The velocity mode satisfies kinematic constraints of three stages \textit{Acceleration-(Max)Speed-Deceleration};
    \item The decision time $\Delta t$ is shorter than the moving time unit.
\end{enumerate}
\end{lemma}

With Lemma~\ref{lemma:Condition}, our HRS modeling and the HR-step-based control admit continuous execution under kinematic constraints:

\begin{theorem}[\textbf{HRS-Continuous Equivalence}]
\label{theorem:Equivalence}
The proposed HRS modeling, combined with HR-step-based decision-making, enables the AGV to run continuously under kinematic constraints.
\end{theorem}

\noindent The proof is described in Appendix~\ref{appendix:proof}. Finally, while we use $c=10$ as an illustrative setting, the specific resolution and mode parameters can be calibrated to different physical systems; the conditions in Lemma~\ref{lemma:Condition} remain unchanged.


\subsection{Alternating Decision--Velocity Control (Alt-DVC)}
\label{subsec:hierarchical_altdvc}
Alt-DVC is the online \emph{synchronous} execution mechanism that alternates macro direction planning and micro velocity adjustment at every HR-step $\Delta t$. At each cycle, the Macro-RL-Routing layer outputs a desired discrete direction $a^t=\pi^*(a\mid s^t)$, and the Micro Velocity-Control layer maps it into feasible continuous motion on the HR-map $M_{c=10}$ under kinematic constraints. Unlike standard RL-MAPF (where agents move one grid per step and effectively ``stop'' at each cell), Alt-DVC must additionally decide (i) \emph{conflict resolution} among moving AGVs with different velocities, and (ii) \emph{when to decelerate} so that an AGV can stop at an appropriate grid center for reorientation. To prevent collisions, Alt-DVC detects \emph{potential collisions} by checking whether an AGV's active zone intersects any other AGV's exclusive zone (Definition~\ref{def:Exclusion}); such conflicts are then resolved online via \textit{priority scheduling}, where higher-priority AGVs may proceed while lower-priority ones treat the higher-priority exclusive zones as obstacles and stop before entering. Priorities are ranked primarily by velocity and then refined by local obstacle-based tie-breaking.
 In the following, we first define collision, then present how to schedule the priority of AGVs to avoid collisions, and finally present how to make the decision on deceleration.
\begin{definition}[\textbf{Potential Collisions Condition}]
\label{def:collisions}
    Given any AGV $i$ with position $l_i^t=(x_i^t,y_i^t)$ at timestep $t$, and each AGV keeping an exclusive zone $ \mathcal{E}_i $ (as defined in Definition \ref{def:Exclusion}), a potential collision is identified occurs when $AGV_i$'s active zone intersects with the exclusive zone of any other $AGV_j$ where $ j \neq i $. 
\end{definition}
It should be noted that the above condition only detects the potential collision, which is subject to being resolved by the following priority scheduling.

\paragraph{\textit{AGV Scheduling Priority}}
\label{subsubsec:method_Velocity_Priority}

We allow the AGV with higher priority to enter the Exclusion Zones of the lower-priority AGVs. Then, AGVs with lower priority must treat the Exclusion Zone of a higher-priority AGV as obstacles, stopping before entering it. This mechanism always allows the higher AGV to keep running, so it can achieve smooth AGV movements without collisions. The AGV priority is determined based on the following criteria:

\begin{enumerate}[label=(\alph*),leftmargin=5mm]
    \item \textit{Velocity-Based Priority}: Collisions at higher velocities can lead to severe property damage, and it takes a longer time and space for a faster AGV to stop, so the priority is first ranked by velocities;
    \item \textit{Obstacle-Awareness Priority}: An AGV with fewer obstacles within its observation window $W$ has greater maneuverability. For example, suppose a set of AGVs is clustered together; then the AGVs on the outline of the cluster should move away to bring more space for the inner AGVs. Thus, the AGVs with fewer obstacles have higher priority;
    \item \textit{Local-Flexibility Priority}: When there is still a tie in the priority, then we shrink the view window incrementally to compare the number of the nearby obstacles until one has fewer obstacles.
\end{enumerate}

Additionally, Alt-DVC triggers deceleration in two situations: \emph{passive deceleration} when required by priority-based collision avoidance, and \emph{active deceleration} when continuing forward would increase the distance to the goal, using a constant-time heuristic comparing the goal distances of the next and second-next grids ($d_x$ vs.\ $d_{xx}$), so that the AGV stops at the next grid and allows the macro layer to update direction.

\noindent\textbf{Case1 Passive Deceleration}: When an AGV has potential collisions (Definition \ref{def:collisions}) with a lower priority, it has to stop passively to avoid collision.

\noindent\textbf{Case2 Active Deceleration}: When the current moving direction does not lead straight to the target location, then the AGV needs to stop and let the RL-routing modify the direction at an appropriate position before continuing toward it. One straightforward way to adjust the deceleration is according to the RL-Routing decision. If the recommended action is not aligned with the current AGV's heading, then the AGV should start decelerating and turn to the recommended heading direction. However, due to the objective of the RL-based method being shorter, the path length finding without considering stopping and turning takes more time. So, frequent stopping and starting will prolong the running time. Another solution adopted by \textit{PRIMAL} is utilizing $A^*$ to compute a path on the clean network as guidance at each step $\delta$. But it has a very high complexity ($O(\ell \times (mk\log mk + 4\cdot mk))$ where $\ell$ is the path length, $mk$ is the grid number, and $4 \cdot mk$ is the edge number) and also a large result path space to choose from. In addition, it is against the flexibility motivation of RL. Therefore, we take a distance-based heuristic to achieve this efficiently and effectively. Specifically, we compare the distance between the second grid ahead and the target, denoted as $ d_{xx} $, with the distance between the first grid ahead and the target, denoted as $ d_x $. If $d_{xx} > d_x $, it means that the AGV is moving away from the target and should be assigned a stop flag, requiring it to stop at the next grid before reorienting itself. The complexity of the distance-based heuristic is reduced to constant.

\section{Inference Efficiency Optimization} 
\label{sec:Optimization}
Given $n$ AGVs, Alt-DVC executes macro direction inference and micro velocity/conflict resolution at each HR-step $\Delta t$. In a straightforward synchronous implementation, all AGVs upload local observations $s^t=(W^t,\vec{u}^t)$ to query the shared policy, and the controller aggregates these local windows to support global priority scheduling, leading to $O(nb^2)$ data per step (and $O(nb^2T)$ over $T$ steps) for observation preparation. Moreover, priority-based conflict checking requires each AGV to be tested against higher-priority ones, resulting in $1+\cdots+n = O(n^2)$ pairwise checks per step (i.e., $O(n^2T)$ over an episode). These costs motivate our inference-efficiency optimization in the following.

\subsection{Asynchronous Decision--Velocity Control (Asy-DVC)}
\label{subsec:Optimization_Asyn}
As analyzed above, the synchronous Alt-DVC inference incurs $O(nb^2T+n^2T)$ time, which may make the decision latency longer than the movement time and thus violate the real-time applicability conditions (Lemma~\ref{lemma:Condition}). To enable real-time execution, we propose \textit{Asy-DVC}, which replaces Alt-DVC with two pruning techniques that reduce both observation and priority-check overhead.

\subsubsection{Observation Pruning}
\label{subsubsec:Optimization_Asyn_Observation}
The key observation is that, due to kinematic constraints, an AGV with nonzero velocity cannot instantaneously change heading; directional decisions only become actionable when the AGV is able to stop and reorient. Therefore, Asy-DVC skips policy inference for \emph{moving} AGVs and only queries the macro policy for the subset of AGVs that are at decision points (typically when $v=0$). Moving AGVs directly proceed with micro-level velocity evaluation (primarily predicting whether to stop ahead), while only a smaller subset performs the $O(b^2)$ observation-and-inference step. Under our velocity-mode design (Fig.~\ref{fig:speedmode}), this reduces the observation cost from $O(nb^2T)$ in Alt-DVC to about $O(\tfrac{1}{4}nb^2T)$ in Asy-DVC (i.e., on average one macro decision every four HR-steps).

\subsubsection{Map-Oriented Priority}
\label{subsubsec:Optimization_Map}
Alt-DVC resolves conflicts via AGV-centric pairwise checking, yielding $O(n^2)$ checks per cycle and $O(n^2T)$ overall. Asy-DVC removes this quadratic factor by using a \emph{map-oriented} index on the coarse map $M_{c=1}$: at each cycle, we insert all AGVs into a grid-indexed priority structure (per grid, keep AGVs sorted by priority); when AGV$_i$ evaluates a candidate motion, it only looks up the relevant destination grid(s) and compares against the top-priority entry, accepting the move if it remains highest-priority locally, otherwise slowing/stopping or deferring to re-planning. This replaces global pairwise comparisons with constant-time grid lookups, reducing priority handling from $O(n^2T)$ to $O(nT)$.

\noindent\textbf{Overall complexity.} Combining the two prunings, Alt-DVC has observation complexity $O(nb^2T)$ and priority complexity $O(n^2T)$, for a total of $O(nb^2T+n^2T)$. In contrast, Asy-DVC reduces these to $O(\tfrac{1}{4}nb^2T)$ (observation pruning) and $O(nT)$ (map-oriented priority), yielding an overall complexity of $O(\tfrac{1}{4}nb^2T+nT)$.

\section{Experiment}
\label{sec:Experiment}
\subsection{Experiment Setup}
\label{subsec:Experiment_Setup}
\subsubsection{Experiment Environment}
All experiments are conducted on a Linux server with two Intel Xeon Platinum 8375C 2.9GHz, 500 GB memory, and two GeForce RTX 4090 GPUs. The proposed model and RL-baselines are built with Pytorch 2.1.1 and Python 3.8.19, and ICBS and OHSMD are implemented in C++ with full optimization. 
\begin{figure}[t!]
    \centering
    \includegraphics[width=\columnwidth]{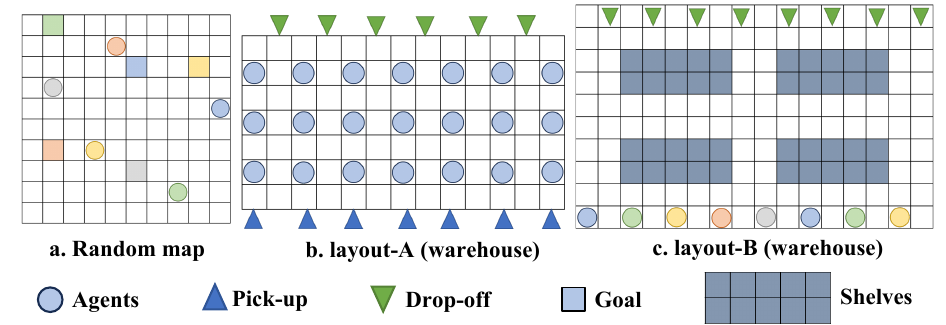}
    \caption{AGV Network Layouts}
    \label{fig:mapstyle}
\end{figure}

\subsubsection{AGV Networks and Experiment Designs}
\label{subsubsec:Experiment_Setup_Map}
We set the high-resolution factor $c=10$ and observation range $b=10$ grids. The RL training hyper-parameters are $\gamma = 0.9$ and $\lambda = 0.5$. The algorithms are tested on the following three types of networks:

\textit{\textbf{Random Map}}: We test on three worlds of size \{40, 80, 160\} with a fleet size of \{64, 128, 256, 512, 1024, 2048\}. The source and goals are generated randomly for each AGV. 
For each feasible parameter combination, we conducted 10 trials with randomly initialized maps, running each trial 10 times and averaging the results for comparison. We excluded certain infeasible scenarios: tests were not conducted with 256 or more AGVs in the world 40, nor with 1024 agents in the world 80. 

\textit{\textbf{Warehouse Layout-A}}: This layout comes from real-life warehouses \cite{wu2025continuous}, where the pickup location is on the top, and the delivery locations are at the bottom. There is no fixed obstacle as it requires more AGVs to operate together. We test on two different network sizes: $G_1$ with 16$\times$214 grids, and $G_2$ with 32$\times$428 grids. The top and bottom rows are designated for 100 pick-up and drop-off points, with one grid separating the adjacent counterparts. The fleet sizes are \{100, 200, 300, 400\}. Some cases are infeasible, like more than 200 AGVs on map $G_1$. The tasks are lifelong in the warehouse, so the AGVs continuously travel between pick-up and drop-off locations to complete 15,000 tasks set together.


To ensure unbiased evaluation, all methods were tested under the same map set and task set.

\subsubsection{Baselines}
\label{subsubsec:Experiment_Setup_Baseline} 
We only compare with the baselines that can scale to our experiment setting. For the graph searching baselines, we compare with i) \textit{ICBS} \cite{boyarski2015icbs}, which is an efficient extension of the classic \textit{CBS} \cite{sharon2015conflict}, and ii) \textit{ODrM}$^*$ \cite{odrm}. For the lifelong tasks, we compare with \textit{OHSMD} \cite{wu2025lifelong}. For the RL-based methods, we compare with the state-of-the-art \textit{PRIMAL} \cite{primal}, as it is the only one that can scale to the large fleet. Finally, we use \textit{HiRAD-Asy} to denote the optimized version with the Asy-DVC, and \textit{HiRAD-Alt} to denote the original Alt-DVC version.

\subsubsection{Metrics}
\label{subsubsec:Experiment_Setup_Metrics}
For the discrete methods, we compare with the following metrics: (1) \textbf{\textit{RunTime (RT)}}: The total code running time to complete the path-planning for all agents; (2)  \textit{\textbf{MakeSpan}} \textbf{(MS (Steps)}: The number of HR-Steps required for the last agent to reach its goal, from start to finish; (3) \textit{\textbf{FlowPath}} \textbf{(FP)}: The total distance traveled by all agents to reach their respective goals
.
To ensure a fair comparison, we multiply the step count of the discrete algorithms by the high-resolution step count required to move one grid in HiRAD, as our method uses higher spatial resolution for discrete movements. 

For the continuous method OHSMD, we evaluate with (1) \textit{\textbf{RunTime}} \textbf{(RT)}: Same as the previous one; (2) \textit{\textbf{Makespan}} \textbf{(MS (time))}: The total time taken for the last agent to reach its goal, from start to finish; (3) \textit{\textbf{FlowTime}} \textbf{(FT)}: The sum of the times taken for all agents to reach their respective targets.
\begin{figure}[t]
    \centering
    \includegraphics[width=1.1\columnwidth]{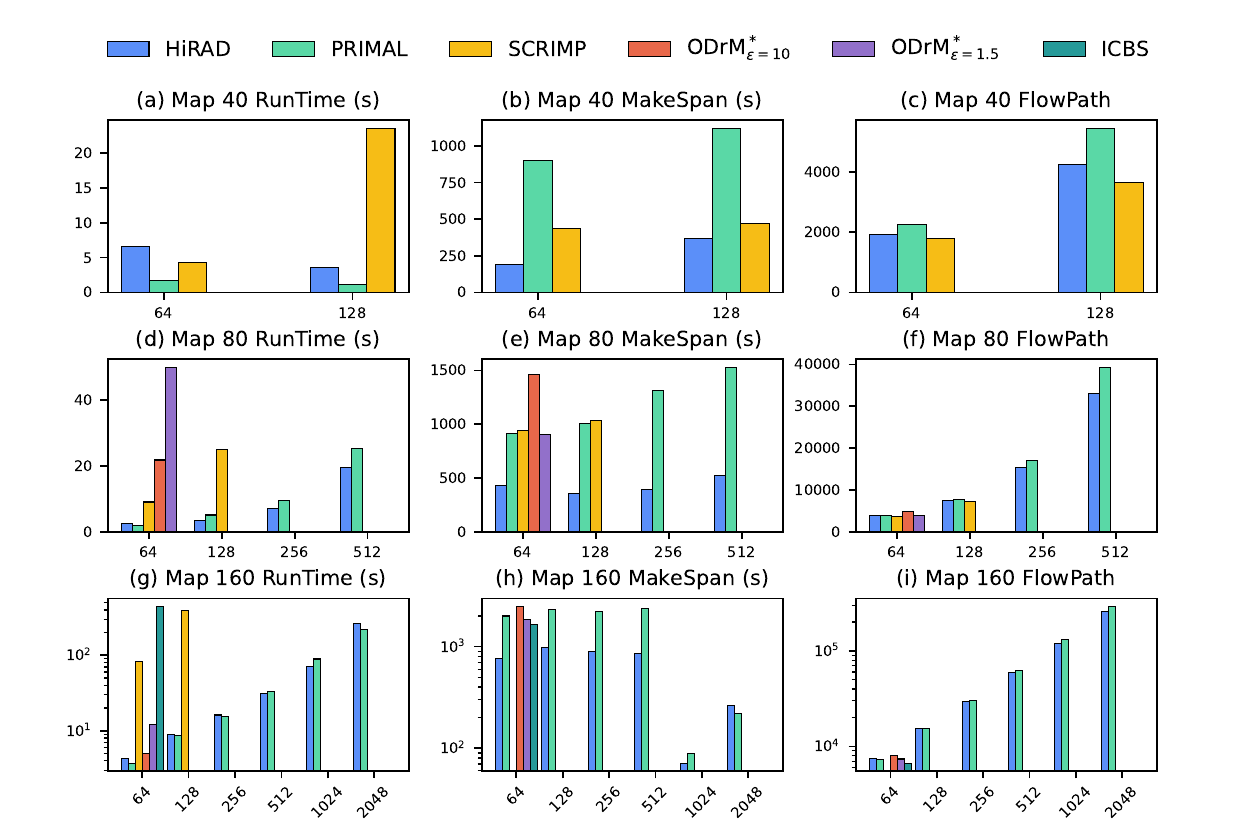}
    \caption{Performance of the AGV Routing Methods with Different Fleet Sizes on Random Maps}
    \vspace*{-4mm}
    \label{fig:performance1}
\end{figure}

\subsection{Random Map Testing}
\label{subsec:Experiment_Random}
Fig. \ref{fig:performance1} presents the performance of the on different map sizes and different fleet sizes. $ \otimes $ indicates that the scenario is not feasible, and `$-$' denotes that no result was obtained within the limited time of 600s. For the learning-based methods, since the process operates through an \textit{observe-decision-execute} cycle, we use the endurance time. Specifically, if no additional AGV reaches its destination within the defined endurance time $\tau_t$, the planning process terminates. At this point, the current completion time \textit{(runtime)} and the number of agents that have reached their destination \textit{(success number)} are recorded. This ensures that the algorithm does not exceed the predefined time limit, while still providing meaningful results within the given $\tau_t$. $\tau_t =5$(s) for the number of agents less than 2048 and $\tau_t =15$(s) for $n=2048$.

\begin{figure}[t]
    \centering
    \includegraphics[width=0.9\linewidth]{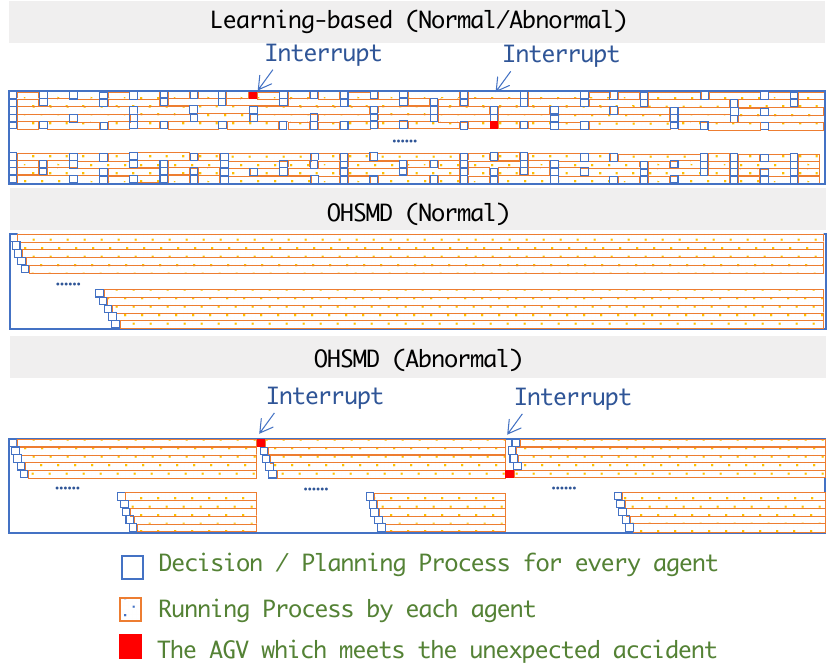}
    \caption{Flexibility Influence Comparison}
    \vspace*{-4mm}
    \label{fig:decision-mode}
\end{figure}

In terms of scalability, the RL-based methods can scale to a much larger fleet than the classical methods. Because \textit{ICBS} and \textit{ODrM$^*$} fail to provide results when the number of AGVs exceeds 64, the table does not display results for 128-2048 AGVs. Additionally, we also tested the performance of \textit{CBS}, but it was unable to provide results for 64 AGV, so its results are not shown in the table. Besides, OHSMD is explicitly designed for Warehouse A, so we do not test it here. In a map size of 160, both ICBS and ODrM$^*$ can only plan paths for up to 64 AGVs. However, PRIMAL and HiRAD can handle path planning for up to 2048 AGVs. 
\begin{table*}[t!]
\centering
\footnotesize
\setlength{\tabcolsep}{2pt}

\begin{minipage}[t]{0.45\textwidth}
\caption{RL Training}
\resizebox{0.9\columnwidth}{!}{%
\begin{threeparttable}    
\begin{tabular}{@{}r|cccccc@{}}
\toprule
\textbf{Map Size}           & \textbf{Small} & \textbf{Large} & \textbf{Large} & \textbf{Large} & \textbf{Large} & \textbf{Large}    \\ \midrule
\textbf{Obstacle Density}    & 0                 & 0                 & 0                 & 0.2               & 0.2                 & 0.2                 \\
\textbf{Reward}              & 25                & 25                & 55                & 55                & 55                & 25--\textgreater{}55   \\
\textbf{IAM}         & N                 & N                 & N                & N                 & Y                 & Y                     \\

\textbf{Curriculum}  & N                 & N                 & N                 & N                 & N                 & Y                      \\
\cmidrule{1-7}
\textbf{SR1} $^\dagger$    & 98\%    & 76\%              & 94\%              & 96\%              & 99\%              & 99\%                  \\
\textbf{Average Steps}       & 33              &  102      &   107                &   115                &  110                 &              99        \\

\textbf{Convergence}     & 4.1h                & 24h               & 18.5h               & 21h             & 17h               & 20h                 \\
\textbf{SR2} $^\dagger$    & 86.8\%             & 71\%              & 90\%              & 95\%              & 99\%              & 99\%                  \\
\bottomrule
\end{tabular}
    \begin{tablenotes}
        \footnotesize
        \item $^{\dagger}$ SR1 is measured by testing different maps under the same settings as the training environment, recording the rate of cases where AGVs successfully reach their targets. SR2 is evaluated on a fixed large map (160×160) with 1024 AGVs, calculating the success rate of AGVs reaching their destinations.
    \end{tablenotes}
\end{threeparttable}
}
\label{tbl:ablation}
\end{minipage}
\hfill
\begin{minipage}[t]{0.32\textwidth}
\caption{Ablation Study}
\resizebox{0.95\columnwidth}{!}{%
\begin{threeparttable}
    
\begin{tabular}{@{}l|c|cc|cc|c@{}}
\toprule
\textbf{\# Agents}      &\multirow{2}{*}{\textbf{Metrics}} &\multicolumn{2}{c}{\textbf{64}}  & \multicolumn{2}{c}{\textbf{256}} & \textbf{1024 }     \\ 
\cmidrule(lr){3-4}\cmidrule(lr){5-6}\cmidrule(lr){7-7}
\textbf{Map Size}  &  & \textbf{40} &\textbf{80}  & \textbf{80}&\textbf{160} & \textbf{160}      \\ \midrule
\multirow{3}{*}{\textbf{HiRAD}}    & \textbf{RT}                 & 3.1                 & 3.8                 & 28               & 34                 & 146                 \\
             & \textbf{MS}                & 103               & 276                & 425                & 737                & 799   \\
        & \textbf{CA}                 & 0                 & 0                & 0                 & 0                & 0                     \\

\cmidrule{1-7}
\multirow{3}{*}{\textbf{-priority}}    & \textbf{RT}    & 3.3              & 3.8             & 22              & 36              & 159                  \\
       & \textbf{MS}              &  176      &   276                &   411                &  730                 &              863        \\

    &\textbf{ CA}                & 4               & 1              & 22             & 7               & 255                 \\
\cmidrule{1-7}
\multirow{3}{*}{\textbf{\begin{tabular}[c]{@{}c@{}}-eclusive \\ Zone\end{tabular}}}     & \textbf{RT}   & 2.6              & 4.3             & 19             & x            & x                  \\
       &\textbf{MS}              &  164      &   286                &   359                &  x                 &              x        \\

    &\textbf{CA}               & 55               & 27               & 500             & 268               & 3937                 \\
\bottomrule
\end{tabular}
\end{threeparttable}
}
\label{tbl:dhc}
\end{minipage}
\hfill
\begin{minipage}[t]{0.2\textwidth}

\caption{Comparison with Cooperative DHC}
 \resizebox{0.85\columnwidth}{!}{%
  \scriptsize
\begin{threeparttable}
    
\begin{tabular}{@{}l|cccc@{}}
\toprule
  & \multicolumn{4}{c}{\textbf{HiRAD}}     \\ \cmidrule(lr){2-5}
  & \textbf{RT} & \textbf{MS} &\textbf{FP} & \textbf{SN}\\
  \midrule
  \textbf{40 }& 1.11 & 192 & 1932 & 64 \\
  \textbf{80} & 2.41 & 426.8 & 3792 & 64\\
  \midrule

 & \multicolumn{4}{|c}{\textbf{DHC}} \\ \cmidrule(lr){2-5}
& \textbf{RT} & \textbf{MS} &\textbf{FP} & \textbf{SN} \\
 \midrule
 \textbf{40} &   8 & 475& 1850 & 64\\
  \textbf{80} &  24.5 & 895& 3630&64 \\
  
\bottomrule
\end{tabular}
\end{threeparttable}
 }
\label{tbl:rltrain}
\end{minipage}
\vspace*{-4mm}
\end{table*}

In terms of the path quality and efficiency, when the fleet size is 64, ICBS achieves the shortest makespan and flow path, providing the highest path quality. However, it requires 446.92s to compute. On the other hand, ODrM$^*$ improves search efficiency by performing dimensionality reduction on the space, with ODrM$^*_{\epsilon=1.5}$ completing the path planning for all AGVs in 12.01s, and ODrM$^*_{\epsilon=10}$ requiring even less time, just 4.98s. ODrM$^*_{\epsilon=10}$ is the fastest classical method in terms of runtime. However, both PRIMAL and HiRAD offer faster runtime speeds than ODrM$^*_{\epsilon=10}$, while the Flow Path does not degrade significantly. Additionally, HiRAD, due to its integration with the velocity control module, can further achieve a shorter makespan while maintaining high efficiency. 

Compared to PRIMAL, HiRAD achieves lower runtime and maintains a lower makespan in most cases. This is because HiRAD allows some AGVs to operate at high velocity, enabling them to complete the pathfinding task much more quickly than when taking discrete, step-by-step actions. On a map of size 40, HiRAD achieves a makespan that is approximately 3-4 times faster than PRIMAL. On a map of size 80, HiRAD is about 2-3 times faster, and on a map of size 160, HiRAD achieves a makespan that is approximately 2 times faster. This demonstrates the efficiency of HiRAD in handling larger-scale pathfinding tasks while maintaining competitive performance in terms of makespan.

It is worth noting that both PRIMAL and HiRAD can successfully plan completion paths for almost all AGVs in all scenarios. However, when the AGV number increases to 2048 on map 160, the number of targets that PRIMAL can successfully reach within the endurance time $\tau_t$ drops significantly. In contrast, HiRAD consistently shows an increase in the number of AGVs that reach their destination, demonstrating its scalability and efficiency for large AGV fleets.

\begin{figure}[t!]
    \centering
    \includegraphics[width=\linewidth]{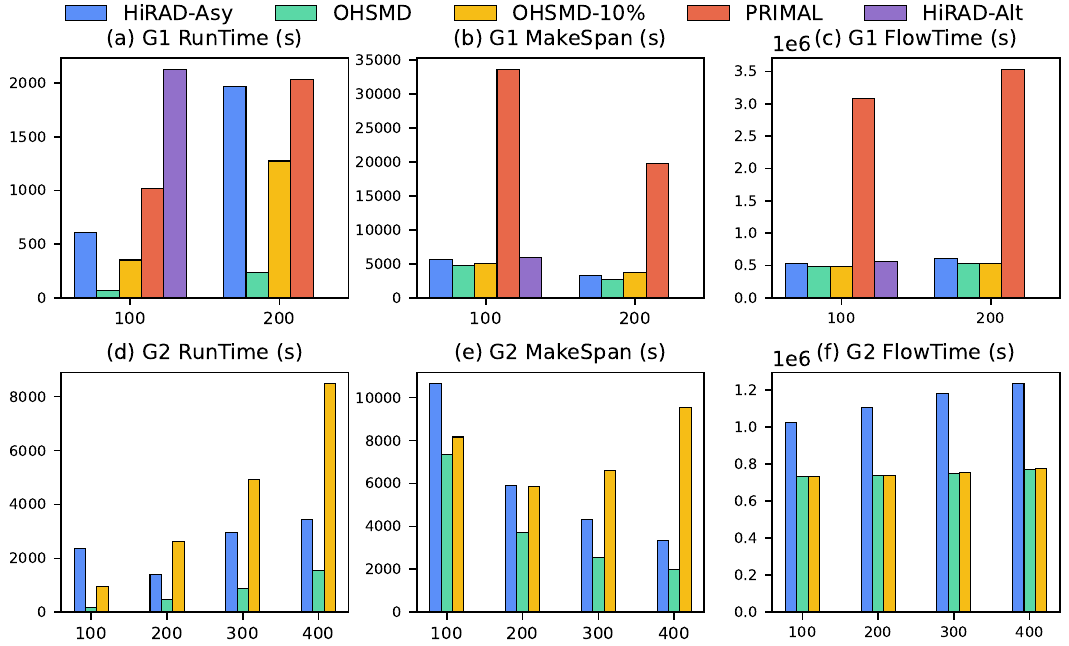}
    \caption{Performance of Lifelong Tasks in Warehouses}
    \label{fig:lifelong}
\end{figure}

\subsection{Warehouse Lifelong Task Testing}
\label{subsec:Experiment_Warehouse}
In this section, we test on the two warehouses with lifelong tasks. 
As shown in Figure \ref{fig:lifelong}, OHSMD performs best on map G1. This is not only because it is implemented in C++ but also because it is specially optimized for G1. However, the primary limitation of OHSMD is its inflexibility: its performance would deteriorate in unexpected situations. Figure \ref{fig:decision-mode} illustrates the running behavior of the RL-based methods and OHSMD. Suppose the horizontal line represents the time spent on the paths of a set of AGVs. The RL-based methods determine the actions step by step before the movement. Therefore, whenever an interruption happens on any one of the AGVs, it does not affect the others. However, OHSMD has to plan the routes for all the AGVs in serial. More importantly, it requires the actual AGVs to follow planned routes spatially and temporally exactly. When any AGV has a little difference in action (due to the network delay, mechanical residual accumulation, or other failures), the whole system has to halt to avoid collisions. Then, the system has to re-plan the remaining routes again based on the current AGVs' locations. Therefore, we further provide a simulation of 10\% of AGVs encountering issues as OHSMD-10\%, which has longer computation times and makespan on the large network than the RL-based methods. This demonstrates the importance of flexibility and the robustness of the proposed methods.

Compared to other learning-based methods, HiRAD-Alt and PRIMAL struggle to scale to a larger number of AGVs and bigger maps due to their slow runtimes. In contrast, the optimized HiRAD-Asy demonstrates higher time efficiency, enabling it to scale effectively with more AGVs and larger maps.

Overall, these results highlight HiRAD’s strong generalization capability and scalability across various warehouse environments. By efficiently handling larger fleet sizes and diverse spatial layouts, HiRAD-Asy proves to be a more robust and adaptable solution for lifelong multi-AGV pathfinding. This adaptability makes it well-suited for real-world applications where dynamic and complex environments demand flexible yet efficient routing strategies.

\subsection{RL-Routing Training Performance}
Since global path planning decisions are driven by reinforcement learning, the effectiveness of the strategy is highly dependent on the diversity of the environment of the training. Table \ref{tbl:rltrain} presents the test performance of models trained with different reward settings, obstacle densities, and action masks on maps of the same size as those used for training. The metrics include success rate (SR1), and average step count (equal to path length). Additionally, the table reports the success rate (SR2) on a specified larger random map with dimensions \(m = 160, n = 1024\). In this section, we define small maps as those with a size between 10 and 40, and large maps as those ranging from 40 to 100.

From Table \ref{tbl:rltrain}, it can be observed when the reward is set to 25 for small maps, SR1, the average steps and the convergence time perform well. However, the trained model does not generalize effectively to large maps, where SR2 shows a significant decline. Setting a small reward directly on the large map prolongs the learning process, increasing the convergence time. Meanwhile, its effect on test maps and multi-agent scenarios remains limited, as both SR1 and SR2 exhibit suboptimal performance.
Increasing the reward for large maps reduces the convergence time while significantly improving SR1 and SR2. This demonstrates that the AGV can successfully find a path in such cases, albeit not optimally. Increasing the obstacle density in the training phase enhances map diversity and raises the learning difficulty for the agent. When the obstacle density reaches 0.2, the convergence time increases; however, the success rate in multi-agent scenarios SR2 improves. 

Given the increased difficulty of exploration in complex maps to reach the goal, we adopt Invalid Action Masking (IAM) during training to filter out non-compliant actions and reduce ineffective explorations, thereby accelerating model convergence. This not only improves training efficiency but also effectively mitigates path redundancy caused by random decision-making, resulting in more optimized path generation. 
Ultimately, we introduced curriculum learning, where small maps are initially trained with small rewards. The difficulty of pathfinding is then gradually increased by expanding the map size, adding more obstacles, and raising the reward. This approach leads to better quality paths with a higher success rate while maintaining a reasonable convergence time.

\section{Conclusion}
\label{sec:Conclusion}
Enabling the AGVs to run efficiently, conflict-free, and flexibly in a real-life warehouse has been a dream for a long time. The existing solutions cannot be used in practice because of their discrete assumption or centralized planning fashion. In this work, the proposed HiRAD framework fills the gap between RL-based routing and practice for the first time. The HRS modeling enables the step-based RL run in continuous environments, the two-layer mechanism reduces the RL search space dramatically, and the asynchronous optimization further improves the inference efficiency such that RL-based AGV can finally run in practice. Experiments validate the effectiveness and efficiency of the proposed framework.

\bibliographystyle{ACM-Reference-Format}
\bibliography{sample-base}

@inproceedings{mnih2016asynchronous,
  title={Asynchronous methods for deep reinforcement learning},
  author={Mnih, Volodymyr and Badia, Adria Puigdomenech and Mirza, Mehdi and Graves, Alex and Lillicrap, Timothy and Harley, Tim and Silver, David and Kavukcuoglu, Koray},
  booktitle={International conference on machine learning},
  pages={1928--1937},
  year={2016},
  organization={PmLR}
}

@INPROCEEDINGS{6907401,
  author={Bnaya, Zahy and Felner, Ariel},
  booktitle={2014 IEEE International Conference on Robotics and Automation (ICRA)}, 
  title={Conflict-Oriented Windowed Hierarchical Cooperative A*}, 
  year={2014},
  volume={},
  number={},
  pages={3743-3748},
}

@inproceedings{yu2013structure,
  title={Structure and intractability of optimal multi-robot path planning on graphs},
  author={Yu, Jingjin and LaValle, Steven},
  booktitle={Proceedings of the AAAI Conference on Artificial Intelligence},
  volume={27},
  number={1},
  pages={1443--1449},
  year={2013}
}

@inproceedings{ferner2013odrm,
  title={ODrM* optimal multirobot path planning in low dimensional search spaces},
  author={Ferner, Cornelia and Wagner, Glenn and Choset, Howie},
  booktitle={2013 IEEE international conference on robotics and automation},
  pages={3854--3859},
  year={2013},
  organization={IEEE}
}

@inproceedings{dai2011artificial,
  title={Artificial intelligence for artificial artificial intelligence},
  author={Dai, Peng and Weld, Daniel and others},
  booktitle={Proceedings of the AAAI Conference on Artificial Intelligence},
  volume={25},
  number={1},
  pages={1153--1160},
  year={2011}
}

@inproceedings{li2019improved,
  title={Improved Heuristics for Multi-Agent Path Finding with Conflict-Based Search.},
  author={Li, Jiaoyang and Felner, Ariel and Boyarski, Eli and Ma, Hang and Koenig, Sven},
  booktitle={IJCAI},
  volume={2019},
  pages={442--449},
  year={2019}
}

@article{shorakaei2016optimal,
  title={Optimal cooperative path planning of unmanned aerial vehicles by a parallel genetic algorithm},
  author={Shorakaei, Hamed and Vahdani, Mojtaba and Imani, Babak and Gholami, Ali},
  journal={Robotica},
  volume={34},
  number={4},
  pages={823--836},
  year={2016},
  publisher={Cambridge University Press}
}

@article{riman2024novel,
  title={Novel Fuzzy Reinforcement Algorithm for Mobile Robot Navigation in Automated Storage},
  author={Riman, Chadi F and Abi-Char, Pierre E},
  journal={International Journal of Mechanical Engineering and Robotics Research},
  volume={13},
  number={2},
  year={2024}
}

@inproceedings{bailey1990automated,
  title={Automated aircraft engine costing using artificial intelligence},
  author={BAILEY, M and OVERTON, K},
  booktitle={26th Joint Propulsion Conference},
  pages={1887},
  year={1990}
}

@article{liu2020prediction,
  title={Prediction, planning, and coordination of thousand-warehousing-robot networks with motion and communication uncertainties},
  author={Liu, Zhe and Wang, Hesheng and Wei, Huanshu and Liu, Ming and Liu, Yun-Hui},
  journal={IEEE Transactions on Automation Science and Engineering},
  volume={18},
  number={4},
  pages={1705--1717},
  year={2020},
  publisher={IEEE}
}

@inproceedings{maoudj2022decentralized,
  title={Decentralized multi-agent path finding in warehouse environments for fleets of mobile robots with limited communication range},
  author={Maoudj, Abderraouf and Christensen, Anders Lyhne},
  booktitle={International Conference on Swarm Intelligence},
  pages={104--116},
  year={2022},
  organization={Springer}
}

@article{damani2021primal,
  title={PRIMAL $ \_2 $: Pathfinding via reinforcement and imitation multi-agent learning-lifelong},
  author={Damani, Mehul and Luo, Zhiyao and Wenzel, Emerson and Sartoretti, Guillaume},
  journal={IEEE Robotics and Automation Letters},
  volume={6},
  number={2},
  pages={2666--2673},
  year={2021},
  publisher={IEEE}
}

@incollection{dhaliwal2020rise,
  title={The rise of automation and robotics in warehouse management},
  author={Dhaliwal, Amandeep},
  booktitle={Transforming management using artificial intelligence techniques},
  pages={63--72},
  year={2020},
  publisher={CRC Press}
}

@inproceedings{wu2025lifelong,
  title={A Lifelong Conflict-Aware AGV Routing System},
  author={Wu, Ruizhong and Zhang, Mengxuan and Wang, Shuxin and Chan, Frodo Kin Sun and Law, Yan Nei and Li, Lei},
  booktitle={Australasian Database Conference},
  pages={447--462},
  year={2025},
  organization={Springer}
}

@article{wu2025continuous,
  title={Continuous Lifelong Conflict-Aware AGV Routing with Kinematic Constraint},
  author={Wu, Ruizhong and Zhang, Mengxuan and Wang, Shuxin and Chan, Frodo Kin Sun and Law, Yan Nei and Li, Lei},
  journal={Proceedings of the VLDB Endowment},
  pages={},
  year={2025},
publisher={VLDB Endowment}
}

@article{wu2026Demo,
  title={A Demonstration of Continuous Lifelong Conflict-Aware AGV Routing with Kinematic Constraints},
  author={Wu, Ruizhong and Zhang, Tianqi and Huang, Yunjie and Li, Lei},
  journal={Proceedings of the VLDB Endowment},
  pages={},
  year={2026},
publisher={VLDB Endowment}
}

@article{custodio2020flexible,
  title={Flexible automated warehouse: a literature review and an innovative framework},
  author={Custodio, Larissa and Machado, Ricardo},
  journal={The International Journal of Advanced Manufacturing Technology},
  volume={106},
  pages={533--558},
  year={2020},
  publisher={Springer}
}

@article{wurman2008coordinating,
  title={Coordinating hundreds of cooperative, autonomous vehicles in warehouses},
  author={Wurman, Peter R and D'Andrea, Raffaello and Mountz, Mick},
  journal={AI magazine},
  volume={29},
  number={1},
  pages={9--9},
  year={2008}
}

@article{bogue2016growth,
  title={Growth in e-commerce boosts innovation in the warehouse robot market},
  author={Bogue, Robert},
  journal={Industrial Robot: An International Journal},
  volume={43},
  number={6},
  pages={583--587},
  year={2016},
  publisher={Emerald Group Publishing Limited}
}

@inproceedings{chen2017decentralized,
  title={Decentralized non-communicating multiagent collision avoidance with deep reinforcement learning},
  author={Chen, Yu Fan and Liu, Miao and Everett, Michael and How, Jonathan P},
  booktitle={2017 IEEE international conference on robotics and automation (ICRA)},
  pages={285--292},
  year={2017},
  organization={IEEE}
}

@inproceedings{ding2018hierarchical,
  title={Hierarchical reinforcement learning framework towards multi-agent navigation},
  author={Ding, Wenhao and Li, Shuaijun and Qian, Huihuan and Chen, Yongquan},
  booktitle={2018 IEEE international conference on robotics and biomimetics (ROBIO)},
  pages={237--242},
  year={2018},
  organization={IEEE}
}

@inproceedings{long2018towards,
  title={Towards optimally decentralized multi-robot collision avoidance via deep reinforcement learning},
  author={Long, Pinxin and Fan, Tingxiang and Liao, Xinyi and Liu, Wenxi and Zhang, Hao and Pan, Jia},
  booktitle={2018 IEEE international conference on robotics and automation (ICRA)},
  pages={6252--6259},
  year={2018},
  organization={IEEE}
}

@INPROCEEDINGS{dhc,
  author={Ma, Ziyuan and Luo, Yudong and Ma, Hang},
  booktitle={2021 IEEE International Conference on Robotics and Automation (ICRA)}, 
  title={Distributed Heuristic Multi-Agent Path Finding with Communication}, 
  year={2021},
  volume={},
  number={},
  pages={8699-8705},
}

@ARTICLE{primal,
  author={Sartoretti, Guillaume and Kerr, Justin and Shi, Yunfei and Wagner, Glenn and Kumar, T. K. Satish and Koenig, Sven and Choset, Howie},
  journal={IEEE Robotics and Automation Letters}, 
  title={PRIMAL: Pathfinding via Reinforcement and Imitation Multi-Agent Learning}, 
  year={2019},
  volume={4},
  number={3},
  pages={2378-2385},
}

@INPROCEEDINGS{odrm,
  author={Ferner, Cornelia and Wagner, Glenn and Choset, Howie},
  booktitle={2013 IEEE International Conference on Robotics and Automation}, 
  title={ODrM* optimal multirobot path planning in low dimensional search spaces}, 
  year={2013},
  volume={},
  number={},
  pages={3854-3859},
}

@inproceedings{shicollision,
  title={Collision-Aware Route Planning in Warehouses Made Efficient: A Strip-based Framework},
  author={Shi, Dingyuan and Zhou, Nan and Tong, Yongxin and Zhou, Zimu and Xu, Yi and Xu, Ke},
  booktitle={2023 IEEE 39th International Conference on Data Engineering (ICDE)},
  year={2023},
  organization={IEEE}
}

@inproceedings{shi2022adaptive,
  title={Adaptive Task Planning for Large-Scale Robotized Warehouses},
  author={Shi, Dingyuan and Tong, Yongxin and Zhou, Zimu and Xu, Ke and Tan, Wenzhe and Li, Hongbo},
  booktitle={2022 IEEE 38th International Conference on Data Engineering (ICDE)},
  year={2022},
  organization={IEEE}
}

@inproceedings{boyarski2015icbs,
  title={Icbs: The improved conflict-based search algorithm for multi-agent pathfinding},
  author={Boyarski, Eli and Felner, Ariel and Stern, Roni and Sharon, Guni and Betzalel, Oded and Tolpin, David and Shimony, Eyal},
  booktitle={Proceedings of the International Symposium on Combinatorial Search},
  volume={6},
  number={1},
  pages={223--225},
  year={2015}
}

@article{sharon2015conflict,
  title={Conflict-based search for optimal multi-agent pathfinding},
  author={Sharon, Guni and Stern, Roni and Felner, Ariel and Sturtevant, Nathan R},
  journal={Artificial Intelligence},
  volume={219},
  pages={40--66},
  year={2015},
  publisher={Elsevier}
}

@inproceedings{li2021eecbs,
  title={Eecbs: A bounded-suboptimal search for multi-agent path finding},
  author={Li, Jiaoyang and Ruml, Wheeler and Koenig, Sven},
  booktitle={Proceedings of the AAAI Conference on Artificial Intelligence},
  volume={35},
  number={14},
  pages={12353--12362},
  year={2021}
}

@article{zhang2018path,
  title={Path planning for the mobile robot: A review},
  author={Zhang, Han-ye and Lin, Wei-ming and Chen, Ai-xia},
  journal={Symmetry},
  volume={10},
  number={10},
  pages={450},
  year={2018},
  publisher={MDPI}
}

@article{liu2023path,
  title={Path planning techniques for mobile robots: Review and prospect},
  author={Liu, Lixing and Wang, Xu and Yang, Xin and Liu, Hongjie and Li, Jianping and Wang, Pengfei},
  journal={Expert Systems with Applications},
  pages={120254},
  year={2023},
  publisher={Elsevier}
}

@article{howden1968sofa,
  title={The sofa problem},
  author={Howden, William E},
  journal={The computer journal},
  volume={11},
  number={3},
  pages={299--301},
  year={1968},
  publisher={Oxford University Press}
}

@article{chien1984planning,
  title={Planning collision-free paths for robotic arm among obstacles},
  author={Chien, Robert T and Zhang, Ling and Zhang, Bo},
  journal={IEEE transactions on pattern analysis and machine intelligence},
  number={1},
  pages={91--96},
  year={1984},
  publisher={IEEE}
}

@inproceedings{stenzel2021automated,
  title={Automated topology creation for global path planning of large AGV fleets},
  author={Stenzel, Jonas and L{\"u}nsch, Dennis and Schmitz, Lea},
  booktitle={2021 IEEE International Intelligent Transportation Systems Conference (ITSC)},
  pages={3373--3380},
  year={2021},
  organization={IEEE}
}

@inproceedings{gomez2020hybrid,
  title={Hybrid topological and 3d dense mapping through autonomous exploration for large indoor environments},
  author={Gomez, Clara and Fehr, Marius and Millane, Alex and Hernandez, Alejandra C and Nieto, Juan and Barber, Ramon and Siegwart, Roland},
  booktitle={2020 IEEE International Conference on Robotics and Automation (ICRA)},
  pages={9673--9679},
  year={2020},
  organization={IEEE}
}

@inproceedings{barer2014suboptimal,
  title={Suboptimal variants of the conflict-based search algorithm for the multi-agent pathfinding problem},
  author={Barer, Max and Sharon, Guni and Stern, Roni and Felner, Ariel},
  booktitle={Proceedings of the International Symposium on Combinatorial Search},
  volume={5},
  number={1},
  pages={19--27},
  year={2014}
}

@inproceedings{wang2016voronoi,
  title={Voronoi-based heuristic for nonholonomic search-based path planning},
  author={Wang, Qi and Wulfmeier, Markus and Wagner, Bernardo},
  booktitle={Intelligent Autonomous Systems 13: Proceedings of the 13th International Conference IAS-13},
  pages={445--458},
  year={2016},
  organization={Springer}
}

@incollection{geraerts2004comparative,
  title={A comparative study of probabilistic roadmap planners},
  author={Geraerts, Roland and Overmars, Mark H},
  booktitle={Algorithmic foundations of robotics V},
  pages={43--57},
  year={2004},
  publisher={Springer}
}

@article{liang2018geometrical,
  title={A geometrical path planning method for unmanned aerial vehicle in 2D/3D complex environment},
  author={Liang, Xiao and Meng, Guanglei and Xu, Yimin and Luo, Haitao},
  journal={Intelligent Service Robotics},
  volume={11},
  pages={301--312},
  year={2018},
  publisher={Springer}
}

@article{dijkstra1959note,
	title={A note on two problems in connexion with graphs},
	author={Dijkstra, Edsger W},
	journal={Numerische mathematik},
	volume={1},
	number={1},
	pages={269--271},
	year={1959},
	publisher={Springer}
}

@article{li2017minimal,
  title={Minimal on-road time route scheduling on time-dependent graphs},
  author={Li, Lei and Hua, Wen and Du, Xingzhong and Zhou, Xiaofang},
  journal={Proceedings of the VLDB Endowment},
  volume={10},
  number={11},
  pages={1274--1285},
  year={2017},
  publisher={VLDB Endowment}
}

@article{hart1968formal,
  title={A formal basis for the heuristic determination of minimum cost paths},
  author={Hart, Peter E and Nilsson, Nils J and Raphael, Bertram},
  journal={IEEE transactions on Systems Science and Cybernetics},
  volume={4},
  number={2},
  pages={100--107},
  year={1968},
  publisher={IEEE}
}

@inproceedings{stentz1994optimal,
  title={Optimal and efficient path planning for partially-known environments},
  author={Stentz, Anthony},
  booktitle={Proceedings of the 1994 IEEE international conference on robotics and automation},
  pages={3310--3317},
  year={1994},
  organization={IEEE}
}

@article{koenig2004lifelong,
  title={Lifelong planning A$^*$},
  author={Koenig, Sven and Likhachev, Maxim and Furcy, David},
  journal={Artificial Intelligence},
  volume={155},
  number={1-2},
  pages={93--146},
  year={2004},
  publisher={Elsevier}
}

@article{koenig2005fast,
  title={Fast replanning for navigation in unknown terrain},
  author={Koenig, Sven and Likhachev, Maxim},
  journal={IEEE Transactions on Robotics},
  volume={21},
  number={3},
  pages={354--363},
  year={2005},
  publisher={IEEE}
}

@article{lavalle1998rapidly,
  title={Rapidly-exploring random trees: A new tool for path planning},
  author={LaValle, Steven},
  journal={Research Report 9811},
  year={1998},
  publisher={Department of Computer Science, Iowa State University}
}

@article{karaman2011sampling,
  title={Sampling-based algorithms for optimal motion planning},
  author={Karaman, Sertac and Frazzoli, Emilio},
  journal={The international journal of robotics research},
  volume={30},
  number={7},
  pages={846--894},
  year={2011},
  publisher={Sage Publications Sage UK: London, England}
}

@inproceedings{geisberger2008contraction,
  title={Contraction hierarchies: Faster and simpler hierarchical routing in road networks},
  author={Geisberger, Robert and Sanders, Peter and Schultes, Dominik and Delling, Daniel},
  booktitle={International workshop on experimental and efficient algorithms},
  pages={319--333},
  year={2008},
  organization={Springer}
}

@article{cohen2003reachability,
  title={Reachability and distance queries via 2-hop labels},
  author={Cohen, Edith and Halperin, Eran and Kaplan, Haim and Zwick, Uri},
  journal={SIAM Journal on Computing},
  volume={32},
  number={5},
  pages={1338--1355},
  year={2003},
  publisher={SIAM}
}

@inproceedings{ouyang2018hierarchy,
  title={When hierarchy meets 2-hop-labeling: Efficient shortest distance queries on road networks},
  author={Ouyang, Dian and Qin, Lu and Chang, Lijun and Lin, Xuemin and Zhang, Ying and Zhu, Qing},
  booktitle={Proceedings of the 2018 International Conference on Management of Data},
  pages={709--724},
  year={2018}
}

@inproceedings{akiba2013fast,
  title={Fast exact shortest-path distance queries on large networks by pruned landmark labeling},
  author={Akiba, Takuya and Iwata, Yoichi and Yoshida, Yuichi},
  booktitle={Proceedings of the 2013 ACM SIGMOD International Conference on Management of Data},
  pages={349--360},
  year={2013}
}

@inproceedings{zhang2021dynamic,
  title={Dynamic hub labeling for road networks},
  author={Zhang, Mengxuan and Li, Lei and Hua, Wen and Mao, Rui and Chao, Pingfu and Zhou, Xiaofang},
  booktitle={2021 IEEE 37th International Conference on Data Engineering (ICDE)},
  pages={336--347},
  year={2021},
  organization={IEEE}
}

@inproceedings{zhang2021efficient,
  title={Efficient 2-hop labeling maintenance in dynamic small-world networks},
  author={Zhang, Mengxuan and Li, Lei and Hua, Wen and Zhou, Xiaofang},
  booktitle={2021 IEEE 37th International Conference on Data Engineering (ICDE)},
  pages={133--144},
  year={2021},
  organization={IEEE}
}

@article{zhang2021experimental,
	title={An experimental evaluation and guideline for path finding in weighted dynamic network},
	author={Zhang, Mengxuan and Li, Lei and Zhou, Xiaofang},
	journal={Proceedings of the VLDB Endowment},
	volume={14},
	number={11},
	pages={2127--2140},
	year={2021},
	publisher={VLDB Endowment}
}

@article{ouyang2020efficient,
  title={Efficient shortest path index maintenance on dynamic road networks with theoretical guarantees},
  author={Ouyang, Dian and Yuan, Long and Qin, Lu and Chang, Lijun and Zhang, Ying and Lin, Xuemin},
  journal={Proceedings of the VLDB Endowment},
  volume={13},
  number={5},
  pages={602--615},
  year={2020},
  publisher={VLDB Endowment}
}

@inproceedings{lissovoi2013runtime,
  title={Runtime analysis of ant colony optimization on dynamic shortest path problems},
  author={Lissovoi, Andrei and Witt, Carsten},
  booktitle={Proceedings of the 15th annual conference on Genetic and evolutionary computation},
  pages={1605--1612},
  year={2013}
}

@article{mohiuddin2016fuzzy,
  title={Fuzzy particle swarm optimization algorithms for the open shortest path first weight setting problem},
  author={Mohiuddin, Mohammad Aijaz and Khan, Salman A and Engelbrecht, Andries P},
  journal={Applied Intelligence},
  volume={45},
  pages={598--621},
  year={2016},
  publisher={Springer}
}


\appendix

\section{Preliminary}
\label{sec:Preliminary} 
We list discrete the AGV network and discrete MAPF problem at first in follow part. Then we introduce the RL-based MAPF modeling in the discrete and continuous environments with a discussion of their pro and cons.
\subsection{AGV Grid Network and Discrete MAPF Problem}
\label{subsec:Preliminary_Grid}
The AGV networks are physically designed as a map $M$ of $m \times k$ square grids of the same size, where $m$ and $k$ are the row and column numbers. We use the coordinate $(i, j)$ to represent the center of a grid $g$ at row $i$ and column $j$, with the upper-left grid's coordinate being $(0, 0)$, and the bottom-right grid's being $(m-1, k-1)$. A fleet of AGVs is denoted as $\mathcal{N}=\{AGV_i\}$. Each vehicle is approximated as a circle in shape, with a diameter that is nearly the same as the grid size. $AGV_i$' location at time $t$ is denoted as $l_i^t=(x_i^t,y_i^t)$. As the AGV's size is roughly the same as a grid, one grid can only be occupied by one AGV at the same time $t$.

The grids on $M$ have the following types: 1) \textit{Static Obstacles} $\mathcal{O}$ that cannot be entered or occupied by any AGV at any time; 2) \textit{Dynamic Obstacles} $\mathcal{L}=\{l_i^t\}$ that are all grids occupied by AGVs; 3) \textit{Vacant Grids} that AGVs can move into; 4) \textit{Start} $s_i$ and \textit{Destination} $d_i$ positions that form the task $\tau_i=\langle s_i,d_i\rangle$ for $AGV_i$, where $s_i = (x^s_i, y^s_i), d_i = (x^d_i, y^d_i)$.

From time $t$ to $t+1$, $AGV_i$ can take five actions $\mathbf{d}a = \{\rightarrow: (0,1), \leftarrow : (0,-1), \uparrow: (-1,0), \downarrow: (1,0), \odot: (0,0)\}$, \ie moving to the neighboring vacant grids or staying at the current grid. Then its next position $l_i^{t+1}=(x_i^{t+1},y_i^{t+1})$ is determined by $(x_i^t,y_i^t)+\mathbf{d}a_{i}$. A path of $AGV_i$ is an ordered sequence of visiting grids $p_i=\langle s_i=l_i^0,l_i^1\cdots,l_i^d=d_i\rangle$. It should be noted that $l_i^t$ could be the same as $l_i^{t+1}$ as the AGV can stay at the same location. The time it takes for an AGV to finish its task is the number of visited grids in the path and is denoted as $|p_i|$. Given a set of tasks, its \textbf{makespan} $max(|p_i|)$ is the time to finish them all, and \textbf{flowtime} $\sum |p_i|$ is the sum of each individual task. Now we are ready to define the discrete MAPF problem as follows:

\begin{definition}[\textbf{Discrete MAPF}]
    \label{def:DiscreteMAPF}
    Given a grid network $M$, a set of tasks $\{\tau_i\}$ and the corresponding AGVs $\mathcal{N}$, the discrete MAPF aims to compute a collision-free path set $\mathcal{P}=\{p_i\}_{i=1}^{|\mathcal{N}|}$ such that their makespan and flowtime are as small as possible.
\end{definition}

\subsection{Discrete RL-MAPF Modeling} 
\label{subsubsec:acnet}
One approach to solving the discrete MAPF problem is to utilize reinforcement learning. It treats each AGV as an agent and makes moving decisions based on the environment. Specifically, AGVs need to 1) observe their \textit{surroundings} to avoid conflicts, and 2) know their \textit{destinations} to plan routes. Specifically, an AGV$_i$'s surrounding $b\times b$ grids are collected into a local view window as a matrix $W_i$, where each matrix element is either 0 (vacant grid) or 1 (obstacles). Its destination indicated by a vector $\vec{u}=\langle\mathbf{d}x,\, \mathbf{d}y,\, \Delta d\rangle \in \mathbb{R}^3$, where $\mathbf{d}x = x_i^d - x_i^t$ and $\mathbf{d}y = y_i^d - y_i^t$ represent the horizontal and vertical displacements, and the scalar $\Delta d = \sqrt{(\mathbf{d}x)^2 + (\mathbf{d}y)^2}$ is the Euclidean distance to the destination. These two components form the observations space $\mathcal{Z}=(W,\vec{u})$, with each AGV$_i$'s denoted as $z_i^t = (W_i^t, \vec{u}_i^t)$ at time step $t$. 



The discrete RL-MAPF problem could be modeled as a Partially Observation Markov Decision Processes (POMDPs), represented by the quintuple \( (\mathcal{S}, \mathcal{A}, P, \mathcal{R}, \pi) \),  
under a \emph{partially observable} setting.
\begin{enumerate}
    \item \textit{State Space} $\mathcal{S}$ is the set of all possible observation configurations $\mathcal{Z}$ that an AGV may encounter. That is, we treat $\mathcal{S} = \mathcal{Z}$, where each distinct observation is regarded as a unique state\footnote{Strictly speaking, in a \emph{partially observable} MDP (POMDP), the observation $z^t \in \mathcal{Z}$ is only a partial projection of the full environment state $s^t \in \mathcal{S}$, since each agent can only perceive its local neighborhood. However, following common practice in reinforcement learning implementations, the observation itself tends to be set as the input state, making $z^t$ and $s^t$ functionally equivalent.}. State $ s^t \in \mathcal{S} $ represents the environment state at time step $t $.
    
    \item \textit{Action Space} $ \mathcal{A} = \operatorname{dom}(\mathbf{d}a) = \{\uparrow, \rightarrow, \downarrow, \leftarrow, \odot\} $ are the decisions that the AGV can execute. The action chosen by AGV$_i$ at the time step $t$ is denoted as $a_i^t \in \mathcal{A}$, and the corresponding state transit is expressed as $s_i^t \xrightarrow{a_i^t}s_i^{t+1}$, with the location of AGV moving from $l_i^t$ to $l_i^t+\mathbf{d}a_i^t$ and other observation information updated after the action is executed. An action is \textit{invalid} if the target grid is out of bounds, blocked by a static obstacle, or occupied by another AGV. Then its transition is either rejected or penalized. In all \textit{valid} cases, any change in position or surroundings caused by taking an action leads to a new state \(s^{t+1} \ne s^t\).
    
    \item \textit{State–Transition Probability} $P(s^{t+1}\mid s^t,a^t)$ denotes the probability that the environment transitions to state $s^{t+1}$ when action $a^t$ is executed in state $s^t$. If the move is rejected, then $s^{t+1}=s^t$.

    \item \textit{Reward Function}~\(\mathcal{R}(s^t,a^t)\) is a designer–specified mapping that assigns a scalar score $r_i^t$ to executing action \(a_i^t\) in state \(s_i^t\) for specific AGV$_i$, thereby encoding task objectives (\eg shorter steps cost) into a numerical feedback signal. Positive values indicate desirable behaviors, such as reaching a goal, whereas negative values penalize undesirable ones, like moving away from the destination and colliding with obstacles or other agents.

    \item \textit{Policy} $\pi : \mathcal{S} \rightarrow \mathcal{A}$ is a mapping from states to actions that specifies the decision rule an AGV follows. It is learned and iteratively refined through reinforcement learning to maximize expected cumulative rewards $\sum_{t=0}^{T} r_i^t$. During training, the policy is continuously updated based on interactions with the environment, gradually converging toward an optimal strategy $\pi^*$. 
\end{enumerate}
Now we are ready to define the RL version of the discrete MAPF problem as follows:
\begin{definition}[\textbf{Discrete RL-MAPF}]
    Given a grid network $M$, a set of tasks $\{\tau_i\}$ and the corresponding AGVs $\;\mathcal{N}$, the discrete RL-MAPF aims to find an optimal policy $\pi^*$ that schedules a collision-free path set $\mathcal{P}$ step by step such that their makespan is as small as possible.
\end{definition}

Because the flow time $\Sigma |p_i|$ is equivalent to the number of action decisions that all AGVs made, the time complexity of inference is linear to $\Sigma |p_i|$. Although it has lower complexity, its result cannot be implemented in real life.

\subsection{Continuous RL-MAPF Modeling}
\label{subsec:Preliminary_PD}
In practice, both time and space are continuous. The moving actions of AGVs must follow kinematic constraints, such as velocity $v$, acceleration $\mathbf{a}$, and deceleration $-\mathbf{a}$: it first accelerates at a constant $\mathbf{a}$ until reaching its maximum velocity $v_{\text{max}}$ or $v < v_{\text{max}}$ when the distance is not sufficient to accelerate to maximum speed, then maintains that velocity and finally decelerates with the deceleration $-\mathbf{a}$ to a complete stop precisely at the center of a grid for turning or keeping stationary. Therefore, path $p_i$ in Continuous MAPF should be reformulated as $ p_i = \langle (x^0, y^0, v^0), \ldots, (x^t, y^t, v^t) \rangle $ that satisfies the following constraints:
\begin{enumerate}[leftmargin=5mm]
    \item The next position $(x^{t+1},y^{t+1})$ is determined by $(x^t,y^t)+\mathbf{d}a^t*v^t $, where $a^t\in \mathcal{A}$, $ \mathbf{d}a = \{\rightarrow: (0,1), \leftarrow : (0,-1), \uparrow: (-1,0) , \downarrow: (1,0), \odot: (0,0)\} $ is the unit vector directional set, and $v^t$ is the velocity;
    \item Velocity changes follow $|v^{t+1} - v^t| \leq \mathbf{a}_{\text{max}}$, where $\mathbf{a}_{\text{max}}$ is the max acceleration/deceleration.
\end{enumerate}

\begin{definition}[\textbf{Continuous RL-MAPF}]
    Given a grid environment $M$, a set of AGVs \( \mathcal{N} \) and a set of tasks $\{\tau_i\}$, continuous MAPF aims to compute a valid path set $ \mathcal{P} = \{p_i\}_{i=1}^{|\mathcal{N}|}, i \in \mathcal{N} $ for them such that the time for each AGV to finish its task is as short as possible with the kinematic constraints.
\end{definition}

However, in continuous environments, theoretically, the decisions can be made at any point in time, which means the potential number of decision points is infinite within any finite interval. Therefore, in practice, continuous time is often discretized for computational purposes, resulting in a large number of very small time steps. The finer the discretization, the more decision points there are, potentially leading to a much larger number of decisions compared to a purely discrete time setting. It works fine in the single-agent environment, but the complexity grows linearly with the number of agents in the multi-agent environment, so no RL-MAPF solution aims to solve this continuous problem.

\section{High-Resolution Spatialtemporal Modeling}
In this section, we supply our Hight-Resolution Spatialtemporal modeling mentioned in Sec.~\ref{sec:HRS}.
\subsection{High-Resolution Map and Step}
\label{appendix:subsec:HRS_Model}
The HRS is composed of two components: a \textit{high-resolution map} that the AGVs navigate spatially, and a \textit{high-resolution step} that the AGVs make decisions temporally. In the following, we define these two components formally.
\begin{definition}[\textbf{High-Resolution HR-Map}]
\label{appendix:def:HRMap}
    Given a grid network $M$, we increase the grid resolution to $ c, c > 1 $ by dividing each grid into $c\times c$ smaller cells to obtain a High-Resolution Map $M_{c}\in \mathbb{R}^{m \times k}$. The original $M=M_{c=1}\in \mathbb{R}^{m \times k}$ is called a Coarse Map.
\end{definition}

As shown in Figure \ref{fig:GridNetwork}, we divide each small grid into $c \times c= c^2$ smaller cells. Each cell's size is shaped in $\frac{1}{c} \times \frac{1}{c}$. 
Note that $M_{c=1}$ and $M_{c>1}$ are two precision definitions for the same network. The only difference between them is the minimum unit. 
Next, we increase the temporal resolution accordingly:
\begin{definition}[\textbf{High-Resolution HR-Step}]
\label{appendix:def:HRStep}
    Given every $\delta$ time that AGV needs to make a decision to move on the coarse map, we divide $\delta$ into $ \Delta t = \frac{\delta}{c}$ HR-Steps such that AGVs need to make in the HR-Map. 
\end{definition}

The reason for dividing the map and steps into the same smallest unit is to guarantee that one $\Delta t$ decision can make an AGV move at least one cell. Meanwhile, it should be noted that the HRS modeling is not equivalent to an $mc\times kc$ cells discrete modeling. This is because, firstly, when an AGV stops or rotates, it still needs to stay at the center of the grids (the $(c/2, c/2)$-cell), not any arbitrary cell. Secondly, in each HR-Step, the AGV can move more than one cell at higher speeds: When an AGV moves an increasing number of cells consecutively, it is accelerating; when it moves an decreasing number of cells, then it is decelerating; it can also moves at a constant number cells when it travels at a constant speed. In this way, the HRS-Modeling is equivalent to continuous modeling, and  we discussed the detailed optimal velocity mapping in Sec. \ref{subsec:hirad_micro}.

\subsection{HRS RL-MAPF Modeling}
\label{appendix:subsec:HRS_Modeling}
Now we are ready to model the RL-MAPF in the proposed HRS environment. Specifically, the action  $a^t = (\mathbf{d}a^t, v^t)$ also combine both direction and velocity. Different from the continuous modeling, the velocity $v^t$ corresponds to the number of cells, with neighboring step's velocity difference being 1 (accelerating), 0 (constant speed), or -1 (decelerating). 

Although it looks similar to the discrete RL-MAPF except for the new velocity constraint, it is more complicated to train from the following aspects: 1) The action space $|\mathcal{A}|$ grows from $ |\mathbf{d}a|$ to $|\mathbf{d}a| \times |v|$. As the discrete resolution of velocity $|v|$ increases, the AGV faces greater difficulty in policy optimization due to the enlarged combinatorial space; 2) The state space dimension grows substantially. If each state of AGV$_i$ is represented as $(s^t_i = (x^t_i, y^t_i, a^t_i, v^t_i)$, the total number of states $|\mathcal{S}|$ grows from $(m\times k \times |\mathbf{d}a| )^{|\mathcal{N}|}$ to $(cm \times ck \times |\mathbf{d}a| \times |v|)^{|\mathcal{N}|}$, which makes the exploration much harder; 3) The reward design becomes more complex. Since AGVs are required to stop precisely at the center of a target grid, the reward function $\mathcal{R}(s^t, a^t)$ depends on the positional error $e = \| (x^t, y^t) - (x^*, y^*) \|$, where $(x^*,y^*)$ is the center coordinates of a grid. Directly controlling velocity may cause overshooting or oscillations, making it harder for $e$ to converge smoothly; 4) The learning process itself is affected. The convergence time of reinforcement learning is approximately proportional to the product of the state and action space sizes, $\mathcal{T}_{\text{conv}} \propto |\mathcal{S}| \times |\mathcal{A}|$. Hence, the enlarged space results in slower convergence and inefficient exploration, reducing the AGV’s ability to learn a stable policy.

In terms of inference, the number of decision-making is $c$ times of the discrete version, while on the other hand, each decision-making time, which has much higher complexity as analyzed above, needs to be faster than $1/c$ of the discrete version.

\section{Macro-RL-Routing Layer} 
\label{appendix:hirad_macro}
This layer aims to direct the AGV to reach the destination without collision on the coarse map.
We give the specifical training process in the following part. and 
\subsection{Macro-RL-Routing Training} 
\label{appendix:hirad_macro_Training}
As analyzed in Section \ref{appendix:subsec:HRS_Modeling}, full-fidelity training with hundreds of AGVs is computationally expensive, so we first learn a single-agent policy on a coarse abstraction of the map $M_{c=1}$ generated randomly. Specifically, the generated maps are constructed such that each grid cell has a probability $p \in \{0,0.1,0.2\}$ of containing an obstacle. The start and destination positions of the AGV are randomly assigned under the constraint that no obstacle occupies them and a valid passage exists between them.

To accelerate convergence, we implement the \textit{curriculum learning}, where the difficulty increases progressively throughout the training process. Initially, the AGV is trained on easier maps with $m=k=10$ and $p = 0$ (no obstacles). As training progresses, the difficulty increases gradually, with the AGV encountering more complex environments, eventually reaching the most challenging configuration where $m=k=100$ (largest training map) and $p=0.2$ (highest obstacle density).

As difficulties increase, the reward setting also changes. This is because a fixed goal reward may not provide sufficient incentive for optimal pathfinding. As Table \ref{tab:reward} shows, An AGV receives a reward of -0.1 for moving up (\(\uparrow\)), down (\(\downarrow\)), left (\(\leftarrow\)), or right (\(\rightarrow\)), while stopping incurs a penalty of -0.3, and collision made the penalty of -2.0. Upon reaching its goal, the AGV is rewarded with a value of +5 to +55, dynamically determined by the training hardness. 
In small maps, a reward of +5 is adequate. But as the map size increases, the accumulated movement penalties (\eg $ -0.1 \times200=-20 $ on map size with $100 \times 100$) can outweigh a small completion reward ($+5$ reward), reducing the AGV's motivation to reach the goal efficiently. Therefore, the goal reward is dynamically scaled from +5 to +55 based on the environment size and difficulty, ensuring that AGVs remain sufficiently incentivized to complete longer paths in larger maps.

\begin{table}[t]
    \centering
    \scriptsize
    \caption{Reward Setting}
    \resizebox{\linewidth}{!}{
    \begin{tabular}{ccccc}
    \toprule
         Action & \makecell{Move \\ ($\uparrow / \downarrow / \leftarrow / \rightarrow$)} &Collision&\makecell{Stay \\($\odot$)}&Find goal\\
         \midrule
         Reward & -0.1 & -2.0 & -0.3 & +5 $\to$ +25 $\to$ +55  \\
    \bottomrule
    \end{tabular}}
    \label{tab:reward}
\end{table}

The employed RL-routing network is built upon the \textit{Actor-Critic AC} framework \cite{mnih2016asynchronous}. The deep neural network was employed to approximate the AGV's policy $\pi$. As illustrated in Figure \ref{fig:framework}, the state $s^t = z^t$ includes the local view window $W^t \in \mathbb{R}^{b\times b} $ with the guidance vector $\vec{u}^t \in \mathbb{R}^3$. They were first encoded and then concatenated to form a feature $x$ at each time step $t$. The sequential features $x$ are processed through the LSTM layer to generate two critical outputs: 1) \textit{Predicted Action Distribution} ($\pi^t(a^t\mid s^t)$):  
A probability distribution over the action space $\mathcal{A}$. Each probability $\pi^t(a^t \mid s^t)$ indicates the likelihood of selecting action $a^t$ under the current policy $\pi^t$ based on the state $s^t$.
2) \textit{Action Value} ($V^t(a^t|s^t)$): A scalar estimate of the expected cumulative rewards when starting from state $s^t$ and taking the action $a^t$ following the current policy $\pi^t$ thereafter. It serves as a baseline or critic signal in policy gradient methods to guide the update toward the optimal policy $\pi^*$.


\subsection{Alternating Decision-Velocity Control}
We identify the following conditions such that the HR-Step-based decision-making RL AGV can run in real life:
\begin{lemma}[\textbf{RL Real-Time Applicability Conditions}]
\label{appendix:lemma:Condition}
Given an HRS-Modeling and its corresponding velocity mapping, the AGVs can run continuously under kinematic constraints in real life only when the following conditions are satisfied:
    \begin{enumerate}[leftmargin=5mm]
        \item The AGV stays at the grid center when it stops;
        \item The velocity mode satisfies kinematic constraints of three stages \textit{Acceleration-(Max)Speed-Deceleration};
        \item The decision time $\Delta t$ is shorter than the moving time unit.
    \end{enumerate}
\end{lemma}
The first condition ensures the AGV movements still follow the grid network design. The second condition ensures the velocity mode follows the physical law and corresponds to the real-life movement. The third condition is the RL efficiency constraint, which is crucial in real-life applications. If it takes a longer time to make decisions than to move, then the decision becomes invalid, and the conflict cannot be avoided among the moving AGVs. Now, we can connect our HRS modeling with the realistic environment:
\begin{theorem}[\textbf{HRS-Continuous Equivalence}]
\label{appendix:theorem:Equivalence}
    The proposed HRS modeling, combined with HR-step-based RL decision-making, enables the AGV to run continuously under kinematic constraints.
\end{theorem}
\begin{proof}
\label{appendix:proof}
    We prove our design satisfies Lemma \ref{lemma:Condition}. Firstly, the Mvx and Mxx all stop at the first and second neighbor grids by default. As for Mvv, when it reaches the speed 0.4, it is in the center of the second grid (1.0 grids and 4 $\Delta t$). Then, it could 1) take another 1.0 grids and 5 $\Delta t$ to stop at the third neighbor grids' center (as Mxx mode), or 2) take another 2 $\Delta t$  to move to the next grid's center place and decide to decelerate or not. If deceleration is chosen, the AGV switches to the Mvz mode and comes to a full stop at the next grid center following a constant deceleration. Therefore, all three modes satisfy Condition 1. Secondly, because the velocity increases and decreases follow the same acceleration parameter and together they satisfy the \textit{Acceleration-(Max)Speed-Deceleration} pattern, it is easy to implement in real life. Finally, as the third condition is about computational efficiency, we will address it in Section \ref{sec:Optimization} by reducing the complexity.
\end{proof}

It should be noted that we use our $c=10$ design as an example, so the actual resolution parameters and velocity mode are subject to change depending on the physical system. However, regardless of the changes, the conditions of Lemma \ref{lemma:Condition} remain unchanged, and the parameters can be obtained through measurement and alignment.

\end{document}